\documentclass{article}
\usepackage{amsmath, amssymb, graphicx, url, fullpage, amsfonts, mathrsfs, amsthm}

\usepackage[round]{natbib}
\usepackage[most]{tcolorbox}
\usepackage{wrapfig}

\usepackage[utf8]{inputenc} 
\usepackage[T1]{fontenc}    
\usepackage{booktabs}       
\usepackage{amsfonts, dsfont}
\usepackage{nicefrac}       
\usepackage{microtype}      
\usepackage{xcolor}         
\usepackage[colorlinks = true, citecolor = blue]{hyperref}

\usepackage{enumitem}
\usepackage{soul}

\usepackage{caption}

\usepackage{algorithm}
\usepackage{algorithmic}

\usepackage{amsmath}
\usepackage{amssymb}
\usepackage{mathtools}
\usepackage{mathrsfs}

\theoremstyle{plain}

\newtheorem{theorem}{Theorem}[section]
\newtheorem{proposition}[theorem]{Proposition}
\newtheorem{lemma}[theorem]{Lemma}

\newtheorem{definition}[theorem]{Definition}
\newtheorem{assumption}[theorem]{Assumption}
\newtheorem{myexample}[theorem]{Example}

\newcommand{\ip}[1]{\langle #1\rangle}

\newcommand{\safetyset}{\mathcal{S}}

\newcommand{\hist}{\mathscr{H}}
\newcommand{\permset}{\widetilde{\mathcal{S}}}
\providecommand{\argmax}{\mathop\mathrm{arg\, max}} 

\newcommand{\ttheta}{\tilde\theta}
\newcommand{\ta}{\tilde{a}}

\newcommand{\spread}{\mathfrak{s}}

\newcommand{\eff}{\mathscr{E}}
\newcommand{\saf}{\mathscr{S}}

\newcommand{\indexU}{\mathsf{U}}

\newcommand{\indi}{\mathds{1}}

\newcommand{\tamat}{\tilde{A}}

\newcommand{\matconf}{\boldsymbol{\mathcal{{C}}}}
\newcommand{\confset}{\mathcal{C}}
\newcommand{\algoname}{\textsc{doss }}
\newcommand{\algonamenospace}{\textsc{doss}}
\newcommand{\algonameparam}{\textsc{doss}($\delta$) }

\newcommand{\tA}{\tilde{A}}
\newcommand{\AI}{A(I)}

\newcommand{\ali}{\alpha(I)}

\newcommand{\tai}{\tA(I)}
\newcommand{\unk}{\mathbf{1}_U}

\newcommand{\xsafe}{x^{\mathsf{s}}}
\newcommand{\msafe}{M^{\mathsf{s}}}

\newcommand{\defispace}{\vspace{-.5\baselineskip}}

\title{Safe Linear Bandits over Unknown Polytopes}

\author{
Aditya Gangrade\textsuperscript{1,2}, Tianrui Chen\textsuperscript{1}, Venkatesh Saligrama\textsuperscript{1}\\ \textsuperscript{1}Boston University, \textsuperscript{2}University of Michigan\\  \texttt{\{gangrade, trchen, srv\}@bu.edu}
}

\date{\vspace{-\baselineskip}}

\begin{document}

\maketitle

\begin{abstract}

    The safe linear bandit problem (SLB) is an online approach to linear programming with unknown objective and unknown \emph{roundwise} constraints, under stochastic bandit feedback of rewards and safety risks of actions. We study the tradeoffs between efficacy and smooth safety costs of SLBs over polytopes, and the role of aggressive {doubly-optimistic play} in avoiding the strong assumptions made by extant pessimistic-optimistic approaches. 

    We first elucidate an inherent hardness in SLBs due the lack of knowledge of constraints: there exist `easy' instances, for which suboptimal extreme points have large `gaps', but on which SLB methods must still incur $\Omega(\sqrt{T})$ regret or safety violations, due to an inability to resolve unknown optima to arbitrary precision. We then analyse a natural doubly-optimistic strategy for the safe linear bandit problem, \algonamenospace, which uses optimistic estimates of both reward and safety risks to select actions, and show that despite the lack of knowledge of constraints or feasible points, \algoname simultaneously obtains tight instance-dependent $O(\log^2 T)$ bounds on efficacy regret, and $\widetilde O(\sqrt{T})$ bounds on safety violations, thus attaining near Pareto-optimality. Further, when safety is demanded to a finite precision, violations improve to $O(\log^2 T).$ These results rely on a novel dual analysis of linear bandits: we argue that \algoname proceeds by activating noisy versions of at least $d$ constraints in each round, which allows us to separately analyse rounds where a `poor' set of constraints is activated, and rounds where `good' sets of constraints are activated. The costs in the former are controlled to $O(\log^2 T)$ by developing new dual notions of gaps, based on global sensitivity analyses of linear programs, that quantify the suboptimality of each such set of constraints. The latter costs are controlled to $O(1)$ by explicitly analysing the solutions of optimistic play.

\end{abstract}

\section{Introduction}\label{sec:intro}

The \textbf{Safe Linear Bandit (SLB) problem:} Consider a linear program \( \max \theta^\top x : Ax \le \alpha \) where the feasible set $\safetyset := \{A x \le \alpha\}$ is known to be a nonempty bounded polytope in $\mathbb{R}^d$, but neither the objective $\theta \in \mathbb{R}^d$, nor the constraint matrix $A \in \mathbb{R}^{m \times d}$ are completely known a priori, and no action known a priori to be safe (i.e., feasible) is available. Instead, a learner sequentially picks actions $x_t$, with the goal of choosing $x_t$ that are effective and safe \emph{in each round}. Learning is enabled through stochastic bandit feedback in the form of a reward signal $R_t = \ip{\theta,x_t} + w_t^R$ and a risk signal $S_t : \mathbb{E}[S_t|x_t] = Ax_t + w_t^S$ where $(w_t^R, w_t^S)$ is a noise process. 

Ideally, we would explore only in $\safetyset,$ but since we do not know it (or any safe point) to start with, some safety violation must necessarily occur over the course of learning. It is natural in many applications to penalise such violation `softly'. With this view, we measure the performance of the learner over $T$ rounds through the \emph{efficacy regret}, $\eff_T,$ and the \emph{net safety violation} $\saf_T$ defined as  \begin{equation}\setlength{\abovedisplayskip}{.3\baselineskip}\setlength{\belowdisplayskip}{.1\baselineskip}
    \eff_T := \sum_{t \le T} \ip{\theta, x^* - x_t})_+ \quad \textit{ and } \quad \saf_T := \sum_{t \le T} (\max_i \ip{a^i,x_t} - \alpha^i)_+, \label{equation:metric_definitions}
\end{equation} 
where $(z)_+ := \max(z,0)$, and $x^*$ is the constrained optimum. These same $\ell_1$ metrics were proposed in the finite-armed setting by \citet{efroni2020exploration} and \citet{chen2022strategies}. The main structural property that makes $\eff_T, \saf_T$ pertinent in the roundwise scenario is that they accumulate only the \emph{positive parts} of the roundwise inefficiency or safety violation. Indeed, since $\eff_T$ sums over $(\ip{\theta, x^* - x_t})_+,$ playing any $x_t$ with better reward than $x^*$ leads to no decrease in it, and instead it increases $\saf_T$ since such an $x_t$ must be infeasible. Conversely, since $\saf_T$ sums the largest roundwise violations, playing a safe but under-effective $x_t$ increases $\eff_T$ but does not reduce $\saf_T$. Thus, the only way to make both $\eff_T$ and $\saf_T$ small is to ensure that most $x_t$s are near-safe \emph{and} near-optimal. We note that the choice of the linear penalty on violations above is just out of convenience: any penalty of the form $f( \max_i (\ip{a^i,x_t} - \alpha^i)_+),$ where $f$ smoothly decays to $0$ near $0^+,$ is amenable to our analysis (\S\ref{appx:alternate_relaxations}).

\paragraph{Motivating Examples.} The interplay of unknown rewards and constraints is a common feature of application domains of bandits. In drug trials, one needs to balance the efficacy of a treatment regimen with its risk of various side-effects (i.e., the probabilities that it induces harmful side-effects); in crowdsourcing, one must balance the cost of completing tasks with the quality of the resulting work; and recommmender systems must balance the click-rate of suggestions with their effects on engagement (such as watch-time or revisiting rates). In such cases, we must enforce the constraint in each round, e.g., completing one task well does not license us to be sloppy on the next. Further, it is nontrivial to find a feasible starting point, since, e.g., this requires knowing worker quality distributions a priori, or knowing which compounds balance the side-effects of active compounds a priori. Nevertheless, soft enforcement is meaningful, e.g., if the risk of a side-effect is slightly over $\alpha$, this only leads to a slight increase in overall numbers of adverse effects realised; and a slight reduction in the mean watch-time is an acceptable price for learning. Thus strong control on $\saf_T$ ensures that in the long run, the system performs arbitrarily close to safety.

\paragraph{Soft Roundwise Enforcement over Polytopes.} We focus on understanding what performance can be achieved while ensuring that $\saf_T = o(T)$. At the first glance, one expects control of the form $\max(\eff_T, \saf_T) = \widetilde{O}(\sqrt{T}),$ which indeed follows from standard techniques (\S\ref{sec:poly_upper_bounds}). However, this question is most interesting in a refined sense: since we are work over a \emph{polytopal} domain,\footnote{While obvious, let us explicitly note here that the problem over polytopal domains is of significant importance, since this corresponds to the ubiquitous questions of linear programming.} prior work on linear bandits tells us that if $\safetyset$ were known, one can derive \emph{instance-dependent} bounds of $O(\log^2 T)$ on $\eff_T,\saf_T$ \citep[e.g.][]{abbasi2011improved}. This paper is concerned with the question \begin{quote} 
\begin{tcolorbox}[colback=gray!20]
Over polytopal domains, is it possible to attain instance-dependent polylogarithmic bounds on $\eff_T$ and $\saf_T$ without knowing $\safetyset$ in advance?
\end{tcolorbox}
\end{quote} 

\paragraph{Our Contributions} approach this by studying the efficacy-safety tradeoffs for SLBs, and by analysing a natural doubly-optimistic method for the same. Concretely, we show that
\begin{itemize}[wide, nosep, leftmargin = 10pt]
    \item \textbf{Simultaneous logarithmic control on $\eff_T$ and $\saf_T$ is impossible.} We show that for any SLB algorithm, there exists an instance \emph{with large `gaps'} on which the algorithm incurs $\max(\eff_T,\saf_T) = \Omega(\sqrt{T}).$ The key property of these instances is the large, i.e., $\Omega(1)$ gap, and due to this gap each instance could be solved $\max(\eff_T,\saf_T) = O(\log^2 T)$ if the feasible set $\safetyset$ were known (\S\ref{sec:lower_bound}). However, a polynomial lower bound arises since the lack of knowledge of $\safetyset$ induces a `\emph{precision barrier},' that is, the fact that no method can \emph{locate} effective and safe actions to precision better than $t^{-1/2}$ after $t$ rounds of play. This same barrier also renders the standard primal approach of analysing polytopal linear bandits via their extreme points ineffective for SLBs(Fig.~\ref{fig:intro_illustration},~left). We further note that the constructed instances are simple enough to embed into any nontrivial set of SLB instances, making the result generic rather than specific to the particular situation we study.
    \item \textbf{Nevertheless, doubly-optimistic (DO) methods can attain $\eff_T = O(\log^2 T)$ and $\saf_T = \widetilde{O}(\sqrt{T})$.} Specifically, we show that these bounds are attained by the DO method \algoname (\S\ref{sec:main_algorithm}), which generalises the finite-armed approach of \citet{efroni2020exploration} and \citet{chen2022strategies}, and has been studied for aggregate enforcement (see below) by \citet{agrawal2014bandits}. \algoname builds an `optimistic' estimate $\permset_t$ of $\safetyset$, and selects actions optimistically over the same. Since these bounds match our lower bounds up to polylog-factors, \algoname is near-Pareto-optimal for SLBs.
    \item \textbf{The aforementioned precision barrier is the sole obstruction to logarithmic bounds.} We argue that in important special cases, \algonamenospace, with either no or mild changes, attains  $\max(\eff_T, \saf_T)  = O(\log^2 T)$. The key property of such settings is an innate way to avoid having to identify good primal actions to arbitrary precision, illustrating that this is the key obstruction in SLBs.
\end{itemize}

\begin{figure}[t]
    \centering
    \includegraphics[width = 0.32\linewidth]{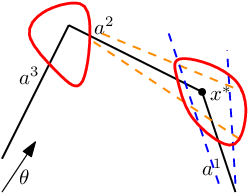}~\hspace{0.1\linewidth}~\includegraphics[width = 0.37\linewidth]{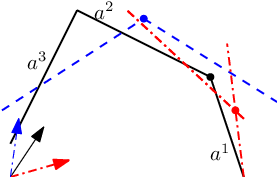} 
    \caption{\footnotesize \textsc{The challenge, and our approach.} \emph{Left.} The usual primal view of linear bandits over polytopes breaks down, since noisy estimates of the unknown $A$ induce a continuum of potential locations for extreme points (red blobs). \emph{Right} Taking a dual linear programming view, we can identify extreme points as arising by saturating $d$ independent constraints. We generalise this view by showing that \algoname plays by saturating noisy versions of $d$ constraints. Poor play can arise from picking the wrong set of constraints (blue), or using a poor estimate for the right set of constraints (red).}
    \label{fig:intro_illustration}
\end{figure}

\paragraph{Technical novelty} of the paper lies in the analysis of \algonamenospace. Since the primal approach for obtaining polylog regret in linear bandits fails, we instead approach the problem through a novel dual analysis, that exploits the fact that extreme points of polytopes can be dually viewed as points saturating $d$ constraints (Fig.~\ref{fig:intro_illustration},~right). We show that this view generalises, i.e., \algoname picks actions by saturating a noisy version of $d$ constraints. This allows us to break the analysis into two threads\\ a)\emph{a combinatorial identification problem of whether the `right' set of constraints is saturated, and \\ b) whether effective points are played when the `optimal' sets of constraints are saturated.}\\ The efficacy loss due to the former is controlled to $O(\log^2 T)$ by developing a novel notion of `dual gap' associated with each `poor' set of constraints, which arise via a global LP sensitivity analysis approach. The second issue is handled via a careful analysis of optimistic play to argue that under mild nondegeneracy assumptions, it cannot play ineffective actions when saturating the `optimal' set of constraints, which controls the efficacy loss due to such play to $O(1)$.

\subsection{Related Work} We briefly describe the two main lines of work on constrained bandits (also see \S\ref{appx:related_work}).  

\paragraph{Hard Roundwise Enforcement.} Instead of the soft sense we study, one can demand roundwise enforcement in a \emph{hard} sense, requiring that with high probability (whp), the constraints always be met, i.e., whp, $\saf_t = 0$. Since this is clearly not possible without knowing a safe point to start with, methods along these lines usually assume a priori knowledge of a point $\xsafe$ in the \emph{interior} of $\safetyset$, i.e., with positive safety margin $\msafe := -\max_i (\ip{a^i, \xsafe} - \alpha^i)$. Given the knowledge of $(\xsafe, \msafe),$ recent lines of work \citep{amani2019linear, moradipari2021safe, pacchiano2021stochastic, afsharrad2023convex, hutchinson2023impact, varma2023stochastic, pacchiano2024contextual} have proposed various `pessimistic-optimisic' (PO) methods for the SLB problem,\footnote{and also safe MDPs \citep[e.g.][]{turchetta2016safe, wachi2020safe, bernasconi2022safe, vaswani2022nearoptimal} } which operate by exploring in the vicinity of $\xsafe$, and build pessimistic estimates of $\safetyset$, over which they act optimistically. While such methods attain the strong safety guarantee of $\saf_T = 0$ whp, the associated costs are significant: $\mathrm{(i)}$ the knowledge of $(\xsafe, \msafe)$ is nontrivial to obtain, and the costs of obtaining the same are not accounted for in this literature,\footnote{Note that the need for a safety margin may make even seemingly simple settings challenging. E.g., if $x$ is the amount of different drugs assigned to a treatment,  one may think that the `no-treatment' drug cocktail $x = 0$ is always `safe', and can serve as $\xsafe$. However, in treatment regimens, it is common that any dose of compound $1$ must be accompanied by a proportional dose of compound $2$ to manage the side-effects induced by compount $1$, i.e, the constraint may be of the form $\ip{ (a_1,-a_2) , x} \le 0,$ in which case $x =0 $ has no safety margin, and so is unusable for PO methods.} and $\mathrm{(ii)}$ the resulting efficacy bounds, $\eff_T = O(d\sqrt{T}/\msafe)$, quantitatively depend on this safety margin.\footnote{We also include a simulation study in \S\ref{section:experiments} that indicates that the safety violations of \algoname are considerably better behaved than the efficacy costs of the PO method \textsc{safe-LTS} \citep{moradipari2021safe}. }

\paragraph{Aggregate Enforcement.} Instead of roundwise metrics, \emph{aggregate constraint enforcement} aims to control $\mathscr{R}_T = \sum \ip{\theta, x^* - x_t}$ and $\mathscr{A}_T = \sum_{t \le T} \max_i (\ip{a^i,x_t} - \alpha_i)$ \citep[e.g.][]{badanidiyuru2013bandits, badanidiyuru2014resourceful, agrawal2014bandits, agrawal2016linear, agrawal2016efficient}. The key difference from the roundwise setting is that there is no nonlinearity in the roundwise penalties in $\mathscr{R}_T,\mathscr{A}_T$. This small change drastically affects allowable behaviour for such problems, e.g., we can ensure $\mathscr{A}_T = o(T)$ while alternating between playing `very unsafe' and `very safe' actions, since the negative costs of the latter cancel the positive costs of the former, but this would instead incur $\saf_T = \Omega(T)$. Of course, $\mathscr{A}_T$ is an inappropriate metric for safety contexts, e.g., treating one patient unsafely cannot be balanced by assigning a placebo to the next. We note that while the analysis of \citet{agrawal2014bandits} \emph{can} be extended to show $(\eff_T, \saf_T) = \widetilde{O}(\sqrt{T})$, we go much beyond this basic observation through our the finer grained upper bounds of $(\log^2 T, \widetilde{O}(\sqrt{T})$, as well as our instance-wise obstructions, which are both novel. We also note that most of the literature on aggregate enforcement explicitly assumes that $x = 0$ is safe, and that the entries of $A$ are positive, which we do not need. Aggregate enforcement remains an active area of research, e.g., `Conservative bandits' \citep[e.g.][]{wu2016conservative} enforce properties of the form $\mathscr{A}_t = O(\sqrt{t})$ for most $t$, and \citet{liu2021efficient} show that given a Slater parameter, one can enforce $\mathscr{A}_t \le 0$ for all $t$ large enough. We also note that most work on constrained MDPs is of this flavour \citep[e.g.][and references therein]{vaswani2022nearoptimal}.

\section{Problem Setup}\label{sec:problem_setup} 

For naturals $a \le b$, let $[a:b] := \{a,\dots, b\}.$ $\ip{\cdot, \cdot}$ and $\|\cdot\|$ denote the inner-product and $\ell_2$-norm in $\mathbb{R}^d$ respectively, and for a matrix $V\succ 0$, $\|z\|_V := \sqrt{\ip{z, Vz}}$. For a $p \times q$ matrix $M$, and a set $\mathsf{S} \subset [1:p],$ $M(\mathsf{S})$ denotes the $|\mathsf{S}| \times q$ submatrix of $M$ preserving rows indexed in $\mathsf{S}$.

\paragraph{Setting.} An instance of polytopal SLB problem is parameterised by an a polytopal region $\mathcal{X} =  \{Bx \le \beta\}\subset \mathbb{R}^d$, a known constraint level vector $\alpha \in \mathbb{R}^U,$ and latent objective $\theta \in \mathbb{R}^d$ and constraint matrix $A \in \mathbb{R}^{U \times d},$ which define the principal LP of relevance. Here, the constraints $\{Bx \le \beta\}$ should be thought of as arising from pre-determined hard limits on $x$.\footnote{e.g., known box constraints in crowdsourcing account for maximum worker capacity, and nonnegativity of work.} For notational succinctness, we will embed these constraints into $(A,\alpha)$ by extending $A$ to lie in $\mathbb{R}^{m \times d}$ for $m = U+K,$ and setting the last $K$ rows of $A$ to $B$, and similarly augment $\alpha$ to include $\beta$. We shall often need the notation $\unk = (1,\cdots, 1, 0,\cdots, 0)$, with $U$ ones, which indicates the unknown constraints. With this notation, the principal LP of interest is \( \max_{x \in \mathcal{X}} \ip{\theta,x} : Ax \le \alpha.\)

\paragraph{Play.} The problem proceeds in rounds, indexed by $t$. For each $t$, we choose an $x_t \in \mathcal{X},$ and receive reward feedback $r_t$ and safety feedback $\{s_t^i\}_{i \in [1:U]}$ that satisfy $r_t = \ip{\theta, x_t} + w_t^R$ and $\smash{s_t^i = \ip{a^i, x_t} + w_t^{S,i}},$ where the various $w_t$s are each subGaussian noise processes, which need not be independent across $i$. The information set of the learner at time $t$ is $\hist_{t-1} := \{ (x_\tau, r_\tau, \{s^i_\tau\}_{i \in [1:U]})_{\tau < t}\},$ and $x_t$ must be adapted to the filtration induced by $\hist_{t-1}$.

\paragraph{Metrics.} We will control the \emph{Efficacy Regret} and \emph{Net Safety Violation} \eqref{equation:metric_definitions}. We reiterate that these have pertinence to the SLB setting because they penalise only the positive parts of roundwise costs.

\paragraph{Assumptions.} We conclude by noting standard assumptions due to \citet{abbasi2011improved}.
    \begin{enumerate}[wide, nosep]
    \item Boundedness: $\|\theta\|\le 1, \|a^i\|\le 1$ for all $i$, and $\mathcal{X} \subset \{\|x\|\le 1\}$ is a bounded polytope.
    \item SubGaussian Noise: $\forall t,$ $w_t:= (w_t^R, \{w_t^{S,i}\}_{i\in[1:U]})$ is conditionally centred and $1$-subGaussian given $\mathcal{F}_{t} := \sigma(\mathcal{H}_{t-1},x_t),$ i.e., \( \forall t,\mathbb{E}[w_t|\mathcal{F}_{t}] = 0, \forall \lambda, \mathbb{E}[\exp(\lambda^\top w_t)|\mathcal{F}_{t}] \le \exp(\|\lambda\|^2/2).\)
\end{enumerate} All subsequent results should be taken to hold only under the above. See \S\ref{appx:standard_assumptions} for more details.

\section{A Doubly Optimistic Algorithm for Safe Linear Bandits}

As previously discussed, our main method of interest is the natural approach of playing optimistically from an optimistic \emph{permissible set} \citep{agrawal2014bandits, efroni2020exploration, chen2022strategies}. We summarise the method, and establish key notation that is used throughout.

\subsection{Confidence Sets and Noise Scales}\label{sec:confidence_sets}

We take the standard approach \citep{abbasi2011improved}. Let the matrix $X_{1:t} = [x_1, \dots, x_t]^\top$ and the vectors $R_{1:t} = [r_1, \dots, r_t]^\top, S^i_{1:t} = [s^i_1, \dots, s^i_t]^\top$ arise by stacking the actions and feedback. The {1-regularised} least squares (RLS) estimate of $\theta, a^i$ using $\hist_{t-1}$ is \[ \hat{\theta}_t=(X_{1:t}^\top X_{1:t}+\lambda I)^{-1}X_{1:t}^\top R_{1:t},\quad \hat{a}_t^i=(X_{1:t}^\top X_{1:t}+\lambda I)^{-1}X_{1:t}^\top S_{1:t}^i. \setlength\abovedisplayskip{3pt}\setlength\belowdisplayskip{3pt}.\] Of course, if $i \in [U+1:m],$ then we do not need to estimate $i$, and we shall just set $\hat{a}_t^i = a^i$ instead. We will collate the $\hat{a}^i_t$s into a matrix $\hat{A}_t$ row-wise. Let us define the signal strength as $V_t:= \sum_{s \le t} x_s x_s^\top + I,$ and for $\delta \in (0,1)$, the $m$-confidence radius as $\sqrt{\omega_t(\delta)} = 1 + \sqrt{\frac12 \log \frac{(U+1)\sqrt{\det V_{t-1}}}{\delta}}.$ The main results are based on the following two concepts, which we explicitly delineate. \defispace

\begin{definition}
    For any time $t$, the \emph{RLS confidence sets} are \[\setlength{\abovedisplayskip}{.3\baselineskip}\setlength{\belowdisplayskip}{.3\baselineskip}\confset_t^{\theta}(\delta) := \{\ttheta: \|\ttheta-\hat\theta_t\|_{V_{t-1}} \le \sqrt{\omega_t(\delta)}\} \text{ and } \matconf_t(\delta) := \{ \tA : \forall \textrm{ rows } i, \|\ta^i - \hat{a}_t^i\|_{V_{t-1}} \le \sqrt{\omega_t(\delta)} \unk\},\] and the \emph{local noise scale} is \(\rho_t(x;\delta)  := 2\sqrt{\omega_t(\delta)}\|x\|_{V_{t-1}^{-1}}.\) \defispace
\end{definition}
The key properties we need are due to \citet{abbasi2011improved}, and are summarised below, and proved in \S\ref{appx:online_linear_reg_background}. We will often drop the dependence of $\confset_t^i(\delta), \matconf_t(\delta),$ and $\rho_t(x;\delta)$ on $\delta$.
\begin{lemma}
\label{lem:noise_scale}
The confidence sets are consistent, i.e., \( \mathbb{P}\left( \forall t,  \theta \in \confset_t^\theta(\delta),  A \in \matconf_t(\delta) \right) \ge 1-\delta.\) Further, under consistency, the noise scale $\rho_t(x;\delta)$ at any $x \in \mathcal{X}$ satisfies $\forall x \in \mathcal{X},$ \[\setlength{\abovedisplayskip}{.3\baselineskip}\setlength{\belowdisplayskip}{.3\baselineskip} \forall \tA \in \matconf_t(\delta), |(\tA - A)x| \le \rho_t(x;\delta)\unk, \quad\textrm{and}\quad \forall \ttheta \in \confset_t^\theta(\delta), |\langle \tilde\theta - \theta, x\rangle| \le \rho_t(x;\delta). \] Finally, for any sequence $\{x_t\}$, $\sum_{s\le t} \rho_s(x_s)^2 = O(d^2 \log^2 t)$ and $\sum_{s \le t} \rho_s(x_s) = \widetilde{O}(\sqrt{d^2 t}).$\defispace
\end{lemma}

\subsection{Doubly-Optimistic Safe Selection}\label{sec:main_algorithm}

\noindent We describe the method, \algoname (Algorithm~\ref{alg:main_scheme}). The key construction herein is the optimistic \emph{permissible set} of points $x$ that are safe according to at least one choice of constraints in $\matconf_t$: \begin{equation}\label{eq:permissible_set} \permset_t(\delta) := \{x : \exists \tilde{A} \in \matconf_t(\delta) \textrm{ s.t. } \tilde{A} x \le \alpha \}. \end{equation} 

\begin{algorithm}[ht]
\caption{Doubly-Optimistic Safe Selection (\algonamenospace) ($\delta$)}\label{alg:main_scheme}
\begin{algorithmic}
\STATE {\bfseries Input:} $\delta \in (0,1)$
\FOR{$t=1,2,\cdots$}
\STATE Construct $\permset_t(\delta)$ as in (\ref{eq:permissible_set}).
\STATE Optimize (\ref{eq:action}) and play $x_t$. 
\STATE Observe $r_{t,x_t}$, $\{s_{t,x_t}^i\}$
\STATE Update $X,R,\{S^i\},V,C$
\ENDFOR
\end{algorithmic}
\end{algorithm}

The set $\permset_t$ consists of all actions that may \emph{plausibly} be safe given $\hist_t$. The arm $x_t$ is selected optimistically from $\permset_t$ as \begin{equation}\label{eq:action} (\tilde\theta_t, x_t) \in \argmax\{ \langle \tilde\theta, x\rangle: {\tilde\theta \in \confset_t^\theta(\delta), x \in \permset_t(\delta)} \} \end{equation}

The optimistic construction of the permissible set is the main distinction between the DO and PO approaches (\S\ref{sec:intro}), which instead work with the pessimistic set $\Pi_t := \{ x : \forall \tA \in \matconf_t, \tA x \le \alpha\} \subset \safetyset$ whp. Instead, $\permset_t(\delta) \supset \safetyset$ whp. Of course, since the known constraints in $A$ are enforced, $\permset_t(\delta) \subset \mathcal{X}.$ 

\section{Warm Up: Polynomial Bounds on Regret and Safety Cost, and Going Beyond}\label{sec:poly_upper_bounds}

An immediate application of the approach of \citet{abbasi2011improved} yields the following basic result, establishing that \algoname is a reasonable procedure. \begin{theorem}\label{thm:poly_upper_bound}
    The actions $\{x_t\}$ of \algonameparam yield, whp, \( \eff_T = \widetilde{O}(\sqrt{d^2 T}) \), and \(\saf_T = \widetilde{O}(\sqrt{d^2 T}). \)    
    
    \noindent{Proof Sketch.} \emph{By Lemma~\ref{lem:noise_scale}, $\forall t, \theta \in \confset_t^\theta, A \in \matconf_t$ whp, and so $x^* \in \permset_t(\delta)$ whp. Thus, \eqref{eq:action} ensures $\langle \ttheta,x_t\rangle \ge \ip{\theta, x^*}.$ But, by the noise-scale characterisation in Lemma~\ref{lem:noise_scale}, $\ip{\ttheta, x_t} \le \ip{\theta,x_t} + \rho_t(x_t),$ and so $\ip{\theta, x^*-x_t}\le \rho_t(x_t)$. On the other hand, since $x_t \in \permset_t,$ there exists some $\tA \in \permset_t : \tA x_t \le \alpha$. But again $\alpha \ge \tA x_t \ge A x_t - \rho_t(x_t)\unk,$ and so $\max_i (\ip{a^i,x} - \alpha^i)_+ \le \rho_t(x_t)$ Consequently, $\eff_T \le \sum_{t \le T} \rho_t(x_t),$ \emph{and} $\saf_T \le \sum_{t \le T} \rho_t(x_t),$ and the bound follows from Lemma~\ref{lem:noise_scale}.}
\end{theorem}

\paragraph{Polytopes to Break Through $\sqrt{T}$?} The above result holds in fact holds over any convex domain without change. However, our domain of interest is linear programming, i.e., $\safetyset$ and $\mathcal{X}$ are polytopes, and thus is much more structured. Indeed, for linear bandits with \emph{known} $\safetyset$, optimistic play yields instance-dependent \emph{logarithmic} regret bounds for large $T$ \citep{abbasi2011improved}. Such results rely on the observation that if $\safetyset$ is known, any action that an optimistic method takes lies in the \emph{finite} set of extreme points of $\safetyset.$ Therefore, $\exists \Delta > 0$ such that for any suboptimal $x_t$, $\ip{\theta, x^*-x_t} \ge \Delta$, and which directly leads to regret bounds of $O(\log^2(T)/\Delta)$.\footnote{The key trick is that $\mathscr{R}_T \le \sum \rho_t(x_t) \mathbf{1}\{\rho_t(x_t) \ge \Delta\} \le \sum \rho_t(x_t)^2/\Delta.$} 

This raises the natural question: \emph{can we also attain logarithmic bounds on $(\eff_T, \saf_T)$ when some of the constraints are unknown?} Answering this will occupy us for the remainder of this paper.

\section{Impossibility of Simultaneous Logarithmic Bounds on Both Efficacy and Safety}\label{sec:lower_bound}

The question we raised in \S\ref{sec:poly_upper_bounds} needs a little care to formulate: since we do not know $\safetyset$, it is unreasonable to expect bounds that scale only with the optimality gap of actions, since unsafe points outside of $\safetyset$ must also be eliminated. We can account for this by also considering the spurious extreme points induced by the bounding polytope $\mathcal{X},$ and consider \begin{align*}\setlength{\abovedisplayskip}{.2\baselineskip}\setlength{\belowdisplayskip}{.2\baselineskip} \mathcal{E} := \{ \textrm{extreme points of } \safetyset\} \cup \{\textrm{extreme points of } \mathcal{X} \}.\end{align*} Now, $\mathcal{E}$ is again a finite set, and for any $x \in \mathcal{E}\setminus\{ x^* \},$ either $x$ is feasible but suboptimal, in which case $\ip{\theta, x^* - x}  > 0$ or it is infeasible, in which case $\max_i (\ip{a^i,x} - \alpha^i) > 0.$ Let us say that an instance is $\Delta$-well separated if the smallest such lower bound is at least $\Delta$. Then note that if we knew $\mathcal{E},$ then it is easy to obtain $O(\Delta^{-1}\log^2 T)$ bounds using the technique described in \S\ref{sec:poly_upper_bounds}. The refined question of interest is: \emph{can we always attain logarithmic efficacy regret and safety violations for well-separated SLB instances?} Surprisingly, the answer to this is negative, as we show in \S\ref{appx:unslacked_regret_lower_bound_proof}.

\begin{theorem}\label{thm:unslacked_regret_lower_bound}
    For every SLB algorithm, there exists a $1/8$-well-separated instance on which the algorithm must incur $\max(\mathbb{E}[\eff_T], \mathbb{E}[\saf_T]) = \Omega(\sqrt{T})$. 
\end{theorem}

\begin{wrapfigure}[9]{r}{0.4\textwidth}
\vspace{-1.25\baselineskip}
\includegraphics[width = \linewidth]{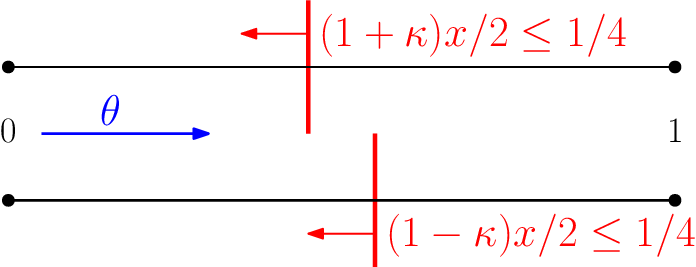}
\caption{\footnotesize An obstruction to logarithmic bounds in safe linear bandits.}
\label{fig:lower_bound}\end{wrapfigure}
\noindent \emph{Proof Sketch}. The obstruction is illustrated in Figure~\ref{fig:lower_bound}. We study the $1$D problem $\max x$ under the known constraints $0 \le x \le 1,$ reward parameter $\theta = 1,$ and the unknown constraint $a x \le 1/4$. Consider the case $a \in \{ \nicefrac{(1 \pm \kappa)}{2}\}$ for $\kappa \le \nicefrac14$. For these instances, $\mathcal{E} = \{0,1, \nicefrac{1}{2(1\pm\kappa)} \}$, and the last point is optimal. Further, $0$ is at least $(2(1\pm \kappa))^{-1} \ge 2/5$-inefficient, and $1$ violates the constraint by $\nicefrac{(1 \pm 2\kappa)}{4} \ge \frac{1}{8},$ and so either instance is $1/8$-well-separated.  

But, no matter the $x_t$s, we cannot estimate $a$ to error better than $\nicefrac{1}{\sqrt{t}},$ and so we cannot eliminate either of $\nicefrac{1 \pm \kappa}{2}$ if $t < \frac{1}{\kappa^{2}}.$ Now, if the truth were $a = \nicefrac{(1-\kappa)}2,$ playing $x_t < \nicefrac{2}{1-\kappa^2}$ incurs inefficacy $\ge 2\kappa$, and conversely if $a = \nicefrac{(1+\kappa)}{2},$ playing $x_t \ge \nicefrac{2}{1-\kappa^2}$ violates safety by $2\kappa$. Thus, at least one of $\eff_T^{\nicefrac{(1+\kappa)}2}$ and $\saf_T^{\nicefrac{(1-\kappa)}2}$ must be $\Omega(\kappa \cdot \min(T,\kappa^{-2})).$ The bound follows by choosing $\kappa= 1/\sqrt{T}$. \hfill$\Box$\vspace{.5\baselineskip}

\paragraph{Impossibility of \emph{instance-dependent} simultaneous logarithmic bounds.} We highlight that the above lower bound scales as $\sqrt{T}$ despite \emph{constant order separation} in the instance. This stands in sharp contrast to existing minimax lower bounds for standard bandits \cite[e.g.][]{shamir2015complexity}, which set $\Delta \sim T^{-\nicefrac12}$ to show $\Omega(\sqrt{T})$ bounds. The barrier to logarithmic control in SLBs is more fundamental, and comes from an inability to refine the precise location of the optimal point, rather than because there are suboptimal points in the noiseless problem that have small gaps. In other words, the issue is one of \emph{precision} rather than one of hardness in the underlying LP, and this makes it impossible to be both very efficient and very safe on all instances. We further observe that the construction is extremely simple, and thus can embed into essentially any class of instances (e.g., by revealing a line that the optimum lies on), and so this issue is pervasive, rather than limited to specific hard cases.

Nevertheless, the result does not preclude that \emph{one} of $\eff_T,\saf_T$ is small. In fact, although they need the extra information $(\xsafe, \msafe)$, we can view PO schemes as saturating this bound, since they achieve $\eff_T = \widetilde{O}(\sqrt{T}), \saf_T = 0.$ We shall show in the subsequent that the DO method \algoname saturates the bound as well, attaining $\eff_T = O(\log^2 T), \saf_T = \widetilde{O}(\sqrt{T}),$ \emph{without this extra information}.

 \paragraph{A dual view, and our approach.} From an analytic point of view, the failure to improve on $\sqrt{T}$ bounds can be seen as a breaking down of the assertion that \emph{in polytopal domains, optimistic methods play on the finite set of extreme points of the polytope}. Indeed, in the SLB scenario, the polytope is not known, and these extreme points are effectively smeared out into sets of diameter $\Omega(t^{-\nicefrac12})$ due to estimation errors in $\hat{A}_t$. Thus, the primal approach to analysing polytopes breaks down.

 As described in \S\ref{sec:intro}, our resolution to this issue lies in the dual view of extreme points of a polytope as points that activate exactly $d$ independent constraints. Due to this, we can view optimism with known $\safetyset$ as activating some $d$ constraints of $\{Ax \le \alpha\}$. This view generalises: we show that there exists some $\tA \in \matconf_t$ such that under \algonamenospace, $x_t$ activates at least $d$ constraints of $\safetyset_{\tA} := \{\tA x \le \alpha\}.$ Naturally, such a set of constraints is a `poor' choice if saturating these constraints for $\{Ax \le \alpha\}$ yields poor or infeasible points. The key idea is that if $I$ is `poor', then the only way \algoname would prefer to activate noisy versions of the constraints in $I$ is if the noise-scale $\rho_t(x_t)$ is large. 
 
 This sets up a two-step attack to control $\eff_T$. First, we use the dual argument above to study a `combinatorial identification' question of whether \algoname finds the `right' set of constraints to saturate. This is addressed by developing new dual notions of gaps for sets of constraints, which arise by an approach reminiscent of the global sensitivity analysis of LPs \citep[][Ch.5]{bertsimas1997introduction}, and is the subject of \S\ref{sec:structural_behaviour}. Secondly, even if the `right' set of constraints are activated, \algoname may play ineffectively due to noisy estimation of $\tA$. Standard arguments (such as \S\ref{sec:poly_upper_bounds}) only yield a $\sqrt{T}$ control on this. Instead, we show that due to the optimism of \eqref{eq:action}, if $t \ge d$ then activating any `optimal' set of constraints yields $x_t : \ip{\theta, x_t - x^*} > 0$, which controls efficacy loss due to such play to $O(1).$ This argument is elementary, but involved, and entails a careful analysis of the structure of \eqref{eq:action} when optimal constraints are activated, as developed in \S\ref{section:regret}, and \S\ref{appx:optimal_BIS_is_good}. 

\section{Structural Behaviour of \algonamenospace, and Noise-Scale Lower Bounds}\label{sec:structural_behaviour}

We proceed to formally define basic index sets, as well as the gaps associated with these index sets, which lead lead to the key consequence that \algoname does not play `suboptimal' index sets too often.

\subsection{Basic Index Sets}\label{sec:BISs}

We begin by formalising `sets of constraints', and `activation' as mentioned in \S\ref{sec:lower_bound}. 

\begin{definition} An \emph{index set} $I$ is a subset of $[1:m].$ Such a set is $I$ is called a \emph{basic index set (BIS)} if $|I| = d$. The set of points that \emph{activate an index set} $I$ is defined as \( \mathcal{X}^I := \{ x \in \mathcal{S} : A(I) x = \alpha(I) \}.\)
\end{definition}
Notice that we demand that activating points are feasible, i.e., $\mathcal{X}^I \subset \safetyset.$ The set $\mathcal{X}^I$ may be empty, or a singleton, or an affine segment. We shall find the following linear-algebraic terminology useful.
 \begin{definition} A BIS $I$ is called $\mathrm{(i)}$ \emph{feasible} if $\mathcal{X}^I \neq \emptyset$ and \emph{infeasible} otherwise; $\mathrm{(ii)}$ {suboptimal} if $x^* \not\in \mathcal{X}^I$ and \emph{optimal} otherwise;  $\mathrm{(iii)}$ \emph{full rank} if the row vectors of $A(I)$ span $\mathbb{R}^d$.
\end{definition}

\begin{wrapfigure}[14]{r}{0.28\linewidth}\vspace{-2.5\baselineskip}
    \centering
    \includegraphics[width = \linewidth]{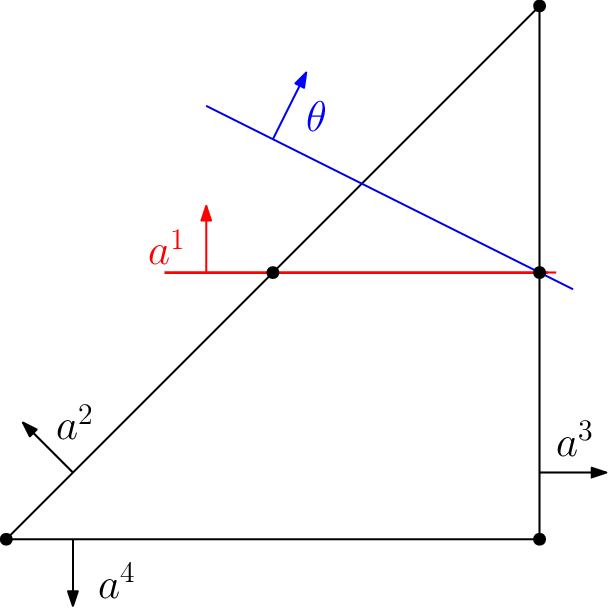}\vspace{-.5\baselineskip}
    \caption{\footnotesize Illustration of Ex.~\ref{example:triangle}. The black lines represent the known constraints, the red line is the unknown constraint, and the blue line is the locus of optimality.}
    \label{fig:triangle_example}
\end{wrapfigure}

$ $ \vspace{-1.5\baselineskip}

\begin{myexample}\label{example:triangle}
     To illustrate these definitions, consider the LP \begin{align*} \max  x_1 + 2x_2 :  \underbrace{x_2 \le 1/2}_{\textrm{unknown}}, \,\,\,\, \underbrace{x_1 \ge x_2, x_1 \le 1, x_2 \ge 0}_{\textrm{known}}.  \end{align*} Foregoing normalisation for clarity, we have $m =4, U = 1$ and the parameters $\theta = (1,2), a^1 = (0,1), a^2 = (-1,1), a^3 = (1,0), a^4 = (0, -1), \alpha = (0.5, 0,1,0)$. There are $\binom{4}{2} = 6$ basic index sets, \begin{align*} I_1 = \{1,2\}, I_2 = \{1,3\}, I_3 = \{1,4\}, \\ I_4 = \{2,3\}, I_5 = \{2,4\}, I_6 = \{3,4\}. \end{align*} Of these, $I_2$ is optimal, and the rest suboptimal, with $x^* = (1,\nicefrac12);$ $I_3$ is rank-deficient while the rest are full-rank;  $I_3$ and $I_4$ are infeasible, while the rest are feasible. 
\end{myexample}

\paragraph{Noisy Activation.} For SLBs, instead of the true constraint matrix $A$, \algoname  must work with noisy estimates of it, the $\tA$s. We extend the notion of BIS activation to handle this fuzziness in constraints. 
\begin{definition} \label{def:noisy_ass} The set of points that \emph{noisily activates} a BIS $I$ at time $t$ is \[ \widetilde{\mathcal{X}}_t^I := \{ x \in \permset_t: \exists \tamat \in \matconf_t, \tamat(I)x = \alpha(I)\}. \]
\end{definition}
Note that $\widetilde{\mathcal{X}}_t^I \subset \permset_t \subset \mathcal{X}$. The main structural result is the following observation. 
\begin{proposition}\label{prop:noisy_association} The actions of \algoname must noisily activate at least one BIS, i.e. $\forall t, \exists I_t : x_t \in \widetilde{\mathcal{X}}_t^{I_t}$.
\end{proposition}

If $x_t$ noisily activates the BIS $I$ at time $t$, we shall say that $I$ is \emph{played} at time $t$. Note that more than one BIS may be played at a time (since $x_t$ can lie in the intersection of many $\widetilde{\mathcal{X}}_t^I$s).

\subsection{Gaps of Suboptimal BISs}\label{sec:gaps} 

We argue that \emph{if \algoname noisily activates a suboptimal BIS at $t$, then the noise scale $\rho_t(x_t;\delta)$ must be large.}  To show this, we develop two \emph{gaps} for suboptimal BISs: the \emph{feasibility gap} and the \emph{efficacy gap}, which respectively exploit the permissibility and optimism of $x_t$. Our results will lower bound $\rho_t(x_t;\delta)$ by the \emph{larger} of these gaps when suboptimal BISs are played. The overall constructions are essentially via a reduction to global linear programming sensitivity analysis. This is necessary: since we do not know the constraints in $A$ or $\theta$, perturbations in this matrix (as represented by $\tamat$) may, and indeed do, cause the optimal $x^*$ to appear suboptimal.

The basic structure we use is the following localisation of $x_t$s played by \algonamenospace, proved in \S\ref{appx:local_box} as a simple consequence of Lemma~\ref{lem:noise_scale}. From here onwards, we shall just write $\rho_t$ instead of $\rho_t(x_t;\delta)$. \begin{lemma}\label{lemma:local_box}
    For $\zeta \in [0,\infty),$ define the \emph{activation polytope} of $I$ at scale $\zeta$ as \[\mathcal{T}(\zeta;I) := \{x : Ax \le \alpha + \zeta \unk , A(I) x \ge \alpha(I) - \zeta \unk(I) \}. \] If the confidence sets are consistent, and if the action of \algoname at time $t$, $x_t,$ noisily activates the BIS $I,$ then $x_t \in \mathcal{T}(\rho_t;I)$, and further, $\ip{\theta, x^* - x_t} \le \rho_t$. 
\end{lemma}

\subsubsection{Intuitive Illustration of Gaps}\label{sec:intuitive_gap}

\begin{wrapfigure}[19]{r}{0.35\linewidth}
    \centering
    \includegraphics[width = \linewidth]{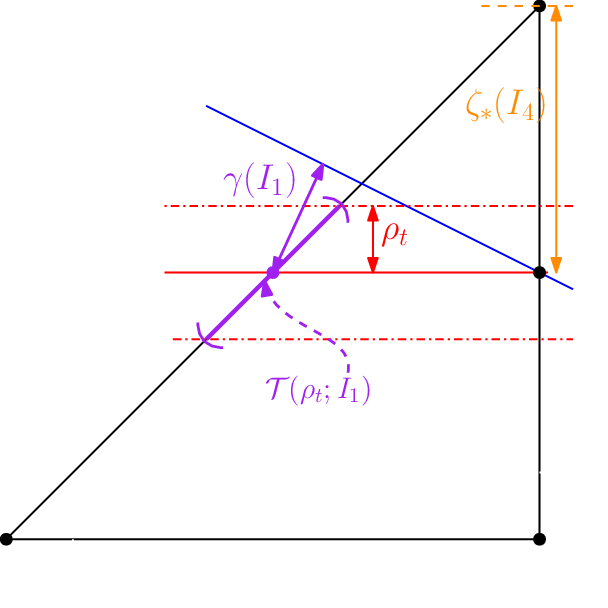}\vspace{-1.5\baselineskip}
    \caption{\footnotesize Illustration of gaps in Ex.~\ref{example:triangle}. $x^{I_1}$ is the purple dot, and the activation polytope $\mathcal{T}(\rho_t;I_1)$ is shown in purple, along with the separation $\gamma(I_1)$. The spread $\spread(I_1)$ is the inner product of the direction in which $\mathcal{T}$ varies and $\theta$. For $I_4,$ the feasibility gap $\zeta_*(I)$ is illustrated geometrically in orange.}
    \label{fig:triangle_example_gaps}
\end{wrapfigure}

To expose the key components that allow \algoname to control the play of suboptimal BISs, we will first consider the feasible, full-rank, and suboptimal BIS $I_1 = \{1,2\}$ in Ex.~\ref{example:triangle}. Due to the full-rank, the constraints of $I$ are activated by a unique point, $x^I$. Since $I$ is suboptimal, there is a positive `efficacy separation' between $x^I$ and $x^*,$ denoted $\gamma(I) := \ip{\theta, x^* - x^I}.$ In our example, $x^{I_1} = (\nicefrac12, \nicefrac12)$, and $\gamma(I_1) = \nicefrac12$.

\paragraph{Efficacy Gap.} Under noisy activation of $I$, the point $x_t$ may depart from $x^I$, but it cannot go too far. Indeed, by Lemma~\ref{lemma:local_box}, $x_t$ must lie in the activation polytope $\mathcal{T}(\rho_t;I),$ which is a skewed $\ell_\infty$-box of scale $\rho_t$ containing $x^I.$ In our example, $\mathcal{T}(\rho_t;I_1) = \{x : x_1 = x_2, x_2 \in \nicefrac12 \pm \rho_t\}.$ This localisation constrains how large $\ip{\theta, x_t}$ can be. Indeed, there exists a constant $\spread(I)$, which we call the spread of $I$, such that $\max_{x \in \mathcal{T}(\zeta;I)} \ip{\theta, x - x^I} \le \zeta \spread(I)$. In effect, $\spread(I)$ is a measure of how well the geometry induced by $I$ near $x^I$ aligns with $\theta$, e.g., for $I_1, \spread(I_1)$ is the inner product between $\theta$ and $(1,1),$ the direction along which $\mathcal{T}$ varies.

Thus, $\ip{\theta, x_t} \le \ip{\theta, x^I} + \rho_t \spread(I).$ But, since $\ip{\theta, x^I - x^*} = - \gamma(I)$, this implies $\ip{\theta, x_t} \le \ip{\theta, x^*} - \gamma(I) + \rho_t\spread(I)$. This lies in tension with Lemma~\ref{lemma:local_box}, which states that $\ip{\theta, x_t} \ge \ip{\theta, x^*} - \rho_t$. Resolving this tension yields the lower bound \( \rho_t \ge \eta_*(I) :=  \gamma(I)/(1+\spread(I)). \) We call the constant $\eta_*(I)$ the \emph{efficacy gap} of $I$. For Ex.~\ref{example:triangle}, $\eta_*(I_1) = \nicefrac18.$

\paragraph{Safety Gap.} It is also possible that $x_t$ noisily activates an infeasible BIS, such as $I_4 = \{2,3\}$ in Ex.~\ref{example:triangle}. In this case, a conflict arises between the inequalities defining the activation polytope $\mathcal{T}(\zeta;I)$: if $I$ is infeasible, then $\mathcal{T}(0;I) = \mathcal{X}^I = \{Ax \le \alpha, A(I) x \ge \alpha(I)\} = \emptyset$, and by right-continuity $\mathcal{T}(\zeta;I)$ is empty for small $\zeta$. Let us define $\zeta_*(I)$ to be the smallest scale at which $\mathcal{T}(\zeta;I)$ is nonempty. Since $x_t \in \mathcal{T}(\rho_t;I)$, it follows that if $x_t$ activates a BIS $I$, then $\rho_t \ge \zeta_*(I)$. We call $\zeta_*(I)$ the \emph{safety gap} of $I$. In Ex.~\ref{example:triangle}, $\mathcal{T}(\zeta;I_4) = \{x : x_1 =1, x_1 = x_2, x_2 \le \nicefrac{1}{2} +\zeta\},$ and so $\zeta_*(I_4) = \nicefrac12.$

\paragraph{Summary.} The above illustrates two basic tensions in selecting suboptimal BISs. If a BIS $I$ is infeasible, then activating it requires that $\rho_t$ dominates its safety gap, and if $I$ is feasible but suboptimal, then activation requires that $\rho_t$ exceeds its efficacy gap. We formalise this concept below.

\subsubsection{Formal Definitions of the Gaps}\label{sec:formal_gap}

We give a unified treatment of the safety and efficacy gaps by analysing a parameterised LP with feasible set determined by the local structure induced by a BIS $I$, as encapsulated in Lemma~\ref{lemma:local_box}. \begin{definition}
    For a BIS $I$, and $\zeta \ge 0,$ the \emph{optimistic LP} at scale $\zeta$ induced by $I$ is defined as \( P(\zeta;I) := \sup \{ \ip{\theta,x} : x \in \mathcal{T}(\zeta;I)\}, \) with the convention that $\sup \emptyset = -\infty$.
\end{definition}
Since by Lemma~\ref{lemma:local_box}, $x_t$ lies in $\mathcal{T}(\rho_t;I)$ if it noisily activates $I$, this yields $\ip{\theta, x_t} \le P(\rho_t;I).$ So, the behaviour of $P(\zeta;I)$ with $\zeta$ let us capture the tensions we illustrated in the previous section. 
\begin{definition}
    We define the \textbf{feasibility gap} of a BIS $I$ as \[ \zeta_*(I) := \inf\{\zeta \ge 0: P(\zeta;I) > -\infty\}. \] We define the \emph{efficacy separation} of $I$ as $\gamma(I) := \ip{\theta, x^*} - P(\zeta_*(I);I),$ and the \emph{spread} of $I$ as \( \spread(I) := \inf\{ C : \forall \zeta \ge \zeta_*(I), P(\zeta;I) \le P(\zeta_*(I)) + C(\zeta -\zeta_*(I)),\) which yield the \textbf{efficacy gap} of $I$, \[\eta_*(I) = \frac{\gamma(I) + \zeta_*(I) \spread(I)}{1 + \spread(I)}. \] 
\end{definition}
The definitions above concretise the quantities described in \S\ref{sec:intuitive_gap}. The key consequence of these definitions is the following `noise-scale lower bound on activating poor BISs,' shown in \S\ref{appx:gaps_make_sense_proofs}.
\begin{lemma}\label{lem:spread_is_finite_AND_noise_scale_lower_bound}
    For any suboptimal BIS $I$, $\max(\zeta_*(I), \eta_*(I)) > 0$. Further, under consistency of the confidence sets, if $x_t$ activates a suboptimal BIS $I$, then $\rho_t(x_t;\delta) \ge \max(\zeta_*(I), \eta_*(I))$. 
\end{lemma}
Note here that the noise-scale needed to play $I$ is driven by the \emph{larger} of the efficacy and safety gap at $I$. This is natural: these quantities measure the `extent' of infeasibility or inefficacy of $I$, and thus the larger one determines the rate at which evidence of the suboptimality of $I$ is accumulated.

\subsection{Gap of the Problem, and Controlling the Play of Suboptimal BISs}\label{section:control_on_play_of_bad_BISs}

In light of Lemma~\ref{lem:spread_is_finite_AND_noise_scale_lower_bound}, the following is natural.
\begin{definition}
    The \textbf{gap of an SLB instance} is defined as $\Gamma:= \min_I \max(\zeta_*(I), \eta_*(I))$. 
\end{definition}
The main  result of this section shows that $\Gamma^{-2}$ bounds how often suboptimal BISs are played. 
\begin{theorem}\label{thm:number_of_bad_bfs}
    Let $\{x_t\}$ denote the actions of \algonameparam on an SLB instance. Then, with probability at least $1-\delta,$ if at any time $t, x_t$ noisily activates a suboptimal BIS, then $\rho_t(x_t;\delta) > \Gamma$. Further, the total number of times suboptimal BISs are played is bounded as \[ \sum_{t} \indi\{ \exists \textrm{suboptimal BIS } I : x_t \in \widetilde{\mathcal{X}}_t^I\} = O\left( \Gamma^{-2} \left(d^2 \log^2T + d\log(T) \log(U/\delta)\right)\right).\]
\end{theorem}
This result, shown in \S\ref{appx:suboptimal_BIS_cannot_be_played_too_often}, implies that most of the time, \algoname plays actions such that the noisy constraints they activate are precisely those that $x^*$ saturates. In other words, while the method may not be able to locate $x^*$ itself with precision better than $O(1/\sqrt{t}),$ it can identify the binding constraints, and,  most of the time, the actions of \algoname focus on activating these constraints.

\section{Controlling Efficacy Regret and Total Safety Violation}\label{section:regret}

We now come to the main results of the paper. The previous section tells us that suboptimal BISs cannot be played too often, effectively controlling a `dual' type of regret. We proceed to translate these results into bounds on the `primal' quantities $\eff_T$ and $\saf_T$. This requires us to account for the times when only optimal BISs ($I$ such that $x^* \in I$) are played. We can control the behaviour of such times under the following weak nondegeneracy condition at the optimum.
\begin{assumption}\label{assumption:non_deg}
    Every optimal BIS (i.e., $I : x^* \in \mathcal{X}^I$) is full-rank. Further, the noise $w_t^S$ is generic in the sense that the probability that it lies in any subspace of less than $d$ dimensions is zero.
\end{assumption}

Note that the condition does not require the uniqueness of the optimum. Instead, nondegeneracy is demanded in the sense that any size $d$ subset of all the constraints that $x^*$ saturates constitutes a full rank BIS. The effect of this is to mainly exclude pathologies, such as the case in $\mathbb{R}^2$ where two identical constraints are placed on the system, and both pass through the optimum (i.e., $(a^i, \alpha^i) = c (a^j, \alpha^j)$ for some pair $i,j$). Notice that in standard linear programmming, such constraints would be eliminated during pre-processing, which we cannot do since we do not know all of the constraints. Nevertheless, since the constraints represent limitations on different safety scores, it is unlikely in practice that these would be linearly dependent. Further, note that Assumption \ref{assumption:non_deg} allows $x^*$ to be degenerate in the sense that it may lie on many more than $d$ constraints. Of course, the genericity of noise is a standard condition, and can be met by adding an arbitrarily small continuous noise to the feedback. The main utility of this assumption is the following result, which is argued in \S\ref{appx:optimal_BIS_is_good}.

\begin{lemma}\label{lemma:optimally_associated_BIS_are_good}
    Under assumption \ref{assumption:non_deg}, if the confidence sets are consistent, $t \ge d+1,$ and the action $x_t$ of \algonameparam is that $x_t$ only noisily activates the optimal BIS, then $\ip{\theta,x_t} \ge \ip{\theta, x^*}$. 
\end{lemma}
In other words, when only the optimal BISs are played, the action $x_t$ cannot be ineffective! The proof relies on using the optimal BIS $I$ to construct a `localised' program that the solutions $(\ttheta_t,x_t)$ and witness $\tA_t$ of \eqref{eq:action} must also optimise. The assumption is used to make this part effective, and in general the same holds if $\theta \in \textrm{row-span}(A(I))$. The final statement then follows through an elementary, but involved, analysis of structure of optimal solutions of this localised program.

Coupling the above with Theorem~\ref{thm:number_of_bad_bfs} yields our main result, shown in \S\ref{appx:main_theorem_proof} \begin{theorem}\label{thm:main_regret_bound}
    Under assumption \ref{assumption:non_deg}, w.p.~$\ge 1-\delta$, the actions of \algonameparam yield \[ \eff_T = O\left( \Gamma^{-1} (d^2 \log ^2 T + d \log T \log(U/\delta)) \right), \textit{ and } \saf_T = \widetilde{O}\left(\sqrt{d^2T (\log^2 T + \log T \log(U/\delta) )}\right).\]
\end{theorem}
In light of Theorem~\ref{thm:unslacked_regret_lower_bound}, we see that up to polylog factors, \algoname saturates the lower bound, with a bias towards minimising the efficacy regret. While the gain in efficacy performance over PO methods is evident, we again stress the advantage in terms of the lack of prior knowledge of a safe ball in $\mathcal{X}.$ We further note that the costs scale logarithmically with the number of unknown constraints, $U$.

\paragraph{Tightness of Dependence on $\Gamma$.} Exploiting a subtle reduction of safe Multi-Armed Bandits problems to SLB problems, we show in \S\ref{appx:logarithmic_lower_bound} that the inverse dependence on $\Gamma$ is necessary. \begin{theorem}
\label{lowerpd} 
Fix a $c \in (0,1)$. For any $\Gamma \le \nicefrac1{16},$ and any method that ensures that in every SLB instance, $\max(\eff_T, \saf_T) = O(T^{1-c}),$ there exists an instance of the SLB problem with gap at least $\Gamma,$ such that \( \liminf \frac{\max(\mathbb{E}[\eff_T], \mathbb{E}[\saf_T])}{\log T} \ge \nicefrac{c}{108}\cdot \Gamma^{-1}.\)
\end{theorem}

\subsection{Improved Safety Performance Under Tolerance}\label{sec:polylog_safety}

While Theorem~\ref{thm:main_regret_bound} is tight in terms of $\saf_T$, given that it achieves polylogarithmic $\eff_T$, the polynomial dependence can nevertheless be considered prohibitive. To improve upon this, we study three concrete scenarios in which this dependence may be improved. At the core, each of these cases relaxes the SLB problem so that the precision barrier discussed in \S\ref{sec:lower_bound} does not arise, thus illustrating that this condition is the sole obstruction to polylogarithmic control on $\saf_T$.

\paragraph{Finite Precision Slack in Constraint Levels.} As a first pass, we may allow for a finite amount of violation of constraints without any penalty, e.g., through the $\varepsilon$-precision metric $\saf_T^\varepsilon := \sum_{t \le T} \max_i (\ip{a^i,x} - \alpha^i)_+ \indi\{ \exists i: \ip{a^i,x} - \alpha^i > \varepsilon\}.$ Such a relaxation is quite pertinent in scenarios such as drug trials or engineering design applications (where $\varepsilon$ can be set to a small factor of $\alpha^i$) or if the $\alpha^i$ are estimated values\footnote{this is quite common: process and measurement variations mean that an exact threshold for the quality of components necessary to ensure safe behaviour is not known, and must usually be fixed empirically.} (where $\varepsilon$ can be the error level in these estimates). In this context, we show in \S\ref{appx:finite_precision_levels} that \begin{theorem}\label{thm:finite_precision_in_levels}
    With probability at least $1-\delta,$ \algonameparam ensures that \emph{simultaneously for every $\varepsilon > 0$} \[ \eff_T = O\left( \Gamma^{-1} {d^2 \log^2 T}\right) \quad \textit{ and } \quad \saf_T = O\left( \varepsilon^{-1} d^2 \log^2 T \right). \] 
\end{theorem} 
The main point of interest in the result above is that it holds \emph{simultaneously} for every value of $\varepsilon$. Indeed, \algoname does not need $\varepsilon$ as a parameter, and it only arises in the analysis. This means that the method adapts to the precision requirements of the domain at hand.  Note further that setting $\varepsilon = T^{-c}$ for $c > \nicefrac12$ yields $\sum_{t \le T} (\max_i \ip{a^i,x_t} - \alpha^i - T^{-c})_+ = \smash{\widetilde{O}(T^{1-c})},$ i.e., as $T \nearrow \infty,$ \algoname rapidly converges towards feasibility, and gains over Theorem~\ref{thm:main_regret_bound} are realised with decaying precision slack. 

\paragraph{Finite Precision in Constraint Parameters.} Rather than treating the precision in the constraint levels, it may be possible that the constraint parameters are restricted to a finite grid. Generically, such a structure arises in settings modeled as integer programs (up to a unit factor), and particular examples include drug discovery \citep[e.g.][]{radhakrishnan2008optimal}, where constraints indicate requirements that a compound binds to certain receptors, and so are naturally binary. We can formalise this by specifying a finite set $\mathsf{P}$ which describes the `grid' that $A$ must lie in. Naturally, we can modify \algoname to exploit this by restricting the construction of $\permset_t(\delta)$ in \eqref{eq:permissible_set} to $\tA \in \matconf_t^{\mathsf{P}} = \matconf_t \cap \mathsf{P}.$ We argue in \S\ref{appx:finite_precision_constraints} that this change implicitly introduces a finite set of possible actions when only optimal BISs are activated, which in turn yields the following result. 

\begin{theorem}\label{thm:finite_precision_in_constraints}
    If the constraint parameters lie in a finite precision set, then there exists a constant $\pi > 0$ such that w.p.~$\ge 1-\delta,$ the actions of \algonameparam satisfy \( \max(\eff_T, \saf_T) = O( \min(\Gamma,\pi)^{-1} d^2 \log^2 T).\) 
\end{theorem} 

\paragraph{Finite Action Spaces.} Finally, if we instead consider the commonly studied case of only having a finite number of possible actions \citep{abbasi2011improved, dani2008stochastic, agrawal2016linear}, then the issues of primal precision do not arise, since we do not need to exactly know the constraints in order to exactly locate any action. If we simply define $\Delta = \min_{\mathcal{X}} \max( \ip{\theta, x^* - x}, \max_i (\ip{a^i,x} - \alpha^i)_+ ),$ then merely employing the techniques described in \S\ref{sec:poly_upper_bounds} yields (see\S\ref{appx:finite_arm}) \begin{proposition}\label{prop:finite_action_regret_bound}
    Over finite actions spaces, with probability at least $1-\delta,$ the actions of \algonameparam ensure that $\max(\eff_T, \saf_T) = O(\Delta^{-1} d^2 \log^2 T).$ 
\end{proposition}

\section{Simulations}\label{section:experiments}
We verify the theoretical study above with simulations over Example~\ref{example:triangle}, and study the relative performance of \algoname and the optimistic-pessimistic method Safe-LTS \cite{moradipari2021safe}. These implementations are based on the following relaxation of Algorithm \ref{alg:main_scheme}.

\paragraph{Computationally Feasible Relaxation.} A well-known barrier to implementing Algorithm \ref{alg:main_scheme} is that even if all constraints were known, the program (\ref{eq:action}) is non-convex \citep{dani2008stochastic}. In our case, this is further complicated by the fact that the set $\permset_t$ needs to be determined. Following \citet{dani2008stochastic}, we approach these issues by constructing \emph{box confidence sets}, i.e., 
\[ \matconf_{t, 1} := \{ \tilde{A}: \forall i, \|(\tilde{a}^i - \hat{a}^i) V_t^{1/2}\|_1 \le \sqrt{ d\beta_{t}}\}. \]
Since $\|\cdot\|_2 \le \|\cdot\|_1 \le \sqrt{d}\|\cdot\|_2, \matconf_{t,1} \subset \matconf_t.$ Further, due to the same equivalence, the $\ell_2$-based analysis persists, up to a blowup of $\sqrt{d}$ in $\rho_t$, and thus running \textsc{doss} with $\matconf_{t,1}$ worsens our bounds from $(d^2 \log^2 T, \sqrt{d^2 T})$ to $(d^3 \log^2 T, \sqrt{d^3 T})$. 

The main advantage of $\matconf_{t,1}$ lies in the fact that the box-confidence sets are polytopes. Due to this, the $\tA_t$ that are active for the optimistic action $x_t$ must lie at the extreme points of these sets. Since each set has only $2d$ extreme points, this allows us to determine $x_t$ by solving $(2d)^{U+1}$ convex programs, which is computationally feasible so long as $U$ is small. Of course, this complexity remains painfully slow as $U$ grows. Finding versions of \textsc{doss} that are computationally practical for a large number of unknown constraints remains an interesting open problem.

\paragraph{Setting.} We implement \algoname on with the $L_\infty$ relaxation above on the instance of Example~\ref{example:triangle} over the horizon $T = 10^4,$ and with the parameters $\lambda = 2, \delta = 1/(4T) = 2.5 \times 10^{-5}.$ The noise in observations is independent and Gaussian, with variance $0.1$. Notice that for this instance, $\Gamma = \nicefrac18$. 

\paragraph{Behaviour of \textsc{doss}.} Our main observation is that \emph{\algoname is very effective, and has well-controlled violations.} Figure~\ref{fig:eff_and_saf} shows the efficacy regret $\eff_t$ and both the arbitrary precision safety violations $\saf_t$ and the finite precision safety violations $\saf_t^{\varepsilon}$ for the value $\varepsilon = 0.05 = 2\Gamma/5$. The simulations validate our main claims of strong efficacy regret control, and well-behaved growth of safety violations. Indeed, observe that the efficacy regret is essentially zero over most of the runs (with rare runs rising to $\eff_{10^4} \approx 100$). This property arises since \algoname very rarely plays suboptimal BISs (see the following discussion and Figure~\ref{fig:suboptimal_BISs}), and when it plays the optimal BIS, it plays a `over-efficient' but unsafe point. Further, the extent of the lack of safety of the actions chosen by \algoname is well-controlled, as seen in the behaviour of $\saf_T$. The finite precision regret shows even stronger control, with growth essentially halted at $t \approx 5000,$ validating the analysis underlying Theorem~\ref{thm:main_regret_bound}.

\begin{figure}[t]
\centering

  \includegraphics[width=.45\linewidth]{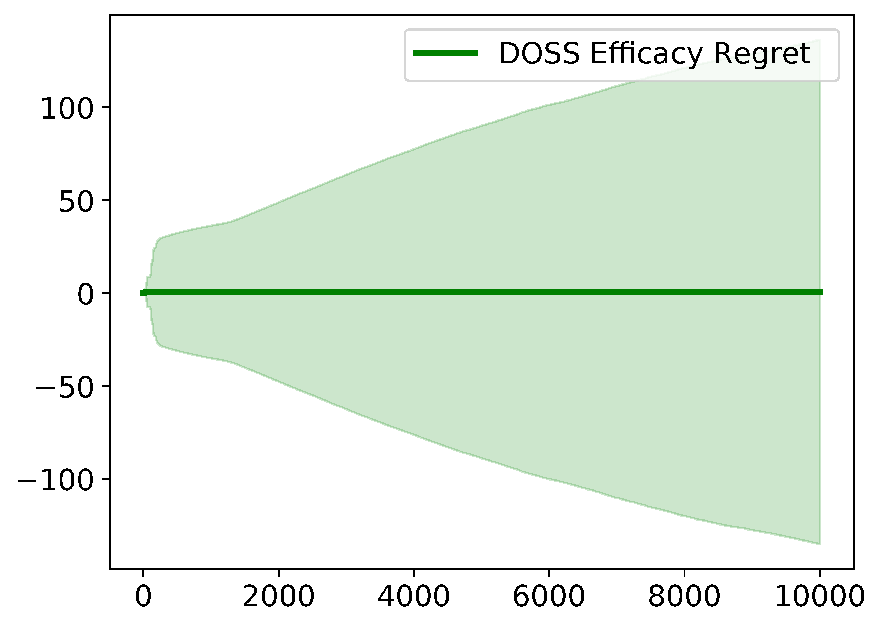}~
  \includegraphics[width=.45\linewidth]{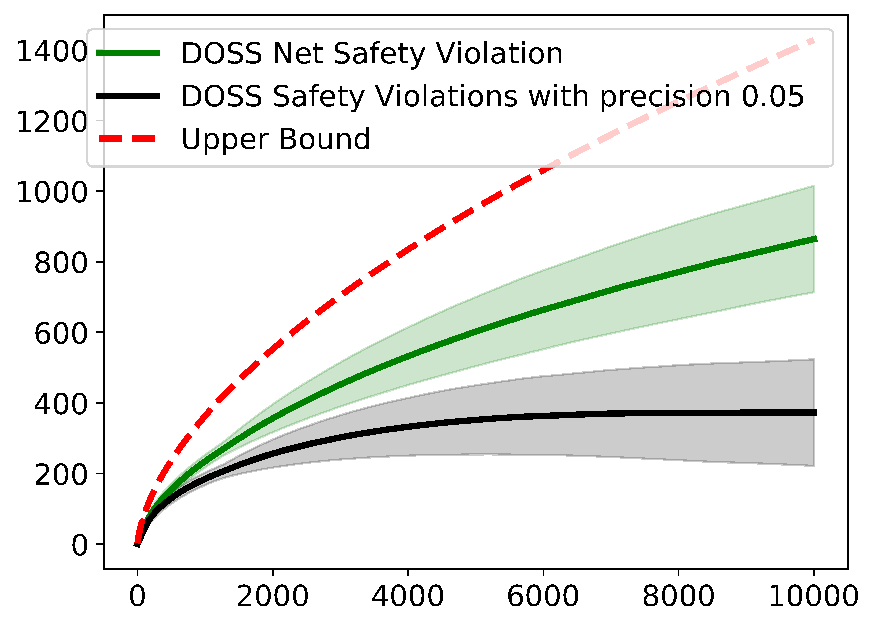}\vspace{-\baselineskip}

\caption{\footnotesize Efficacy Regret and Safety Violation of \algoname. We plot averages and one standard deviation confidence regions over 30 runs for $\eff_T$ (left) and both $\saf_t$ and $\saf_t^{0.05}$ (right). We also plot the upper bounds we show in the latter to contextualise the observations. Observe that the efficacy regret is marginal: the mean is essentially $0,$ and the variance limited. Further, observe that the growth of the net safety violation $\saf_t$ is well-controlled, and lies far below the bounds of \S\ref{section:regret}. Further, the finite precision violations show a strong flattening, as is expected from Theorem~\ref{thm:main_regret_bound}.}
\label{fig:eff_and_saf}\vspace{-\baselineskip}
\end{figure}

\paragraph{\algoname rarely activates suboptimal index sets.} In Figure~\ref{fig:suboptimal_BISs}, we plot the number of times that \algoname noisily activates a suboptimal BIS, i.e., any index set other than $I_2 = \{1,3\}$. The main observation is that this occurs very rarely: indeed, over the horizon of $10^4$, most runs do not activate suboptimal BISs more than 100 times. This is far below the upper bound of Theorem~\ref{thm:number_of_bad_bfs}. 

\begin{figure}[t]
\centering
  \includegraphics[width=.5\linewidth]{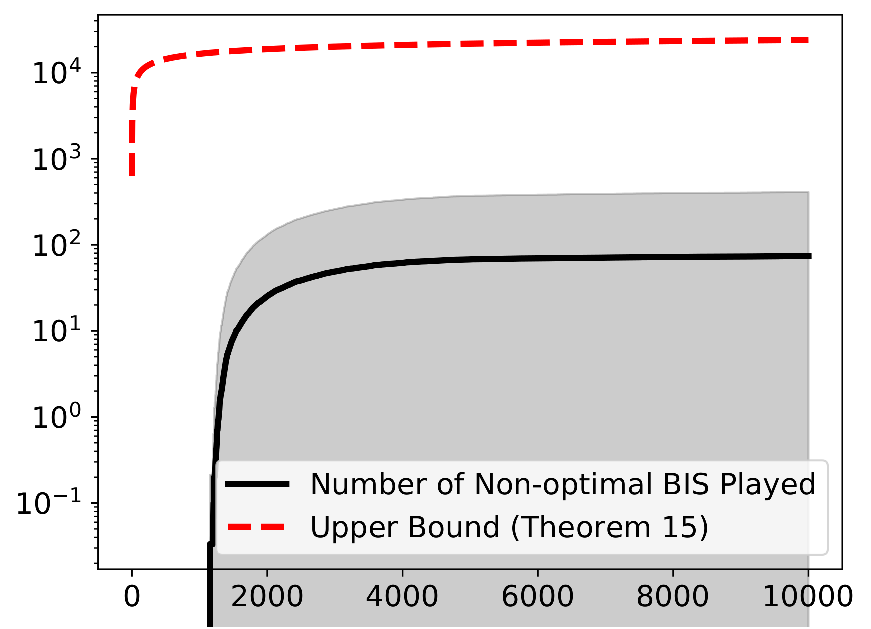}
  \caption{\footnotesize Suboptimal BIS activation by \algoname in the instance of Example~\ref{example:triangle}. Observe that such activation is very rare, typically far less than $1\%$ of the times, and the growth is essentially flat.}
  \label{fig:suboptimal_BISs}
\end{figure}

\paragraph{\algoname Compares Favourably with Pessimistic-Optimistic Methods.}

To contextualise our method, we also implement the PO-method \emph{safe-LTS} due to \cite{moradipari2021safe} in the instance of Example~\ref{example:triangle}. Instead of the optimistic permissible set $\permset$, safe-LTS constructs a pessimistic set $\Pi_t = \{x : \forall \tA \in \matconf_t, \tA x \le \alpha\}$.  Note that with high probability, all points in $\Pi_t$ must be safe. The method then selects actions optimistically, in this case by exploiting Thompson sampling. Naturally, this method requires the knowledge of a safe point with margin to being with, and we supply the point $\xsafe = (0,0)$ to the method, which has the (large) margin $\msafe = \nicefrac 12$.

\begin{figure}[t]
\centering
  \includegraphics[width=.45\linewidth]{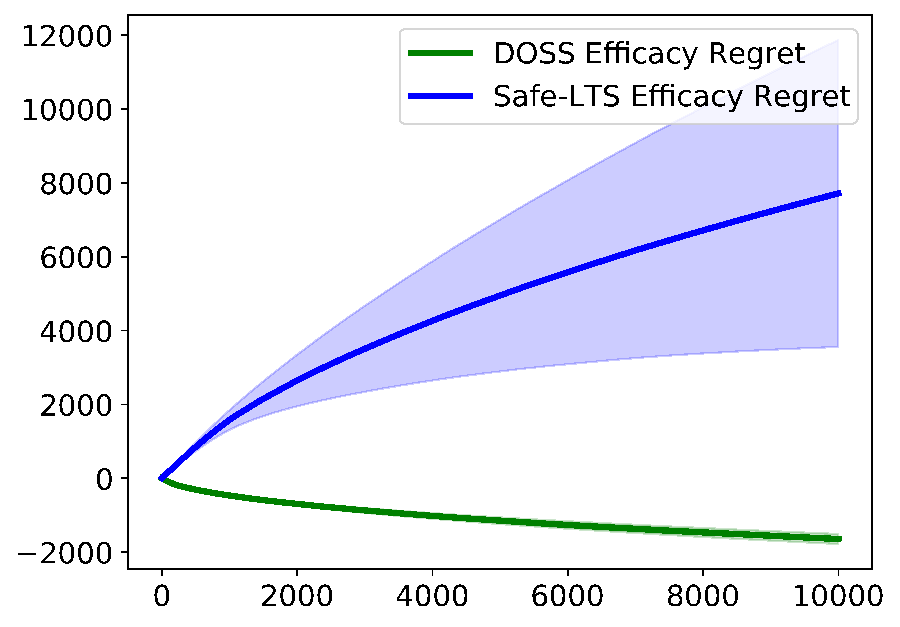}
  \includegraphics[width=.45\linewidth]{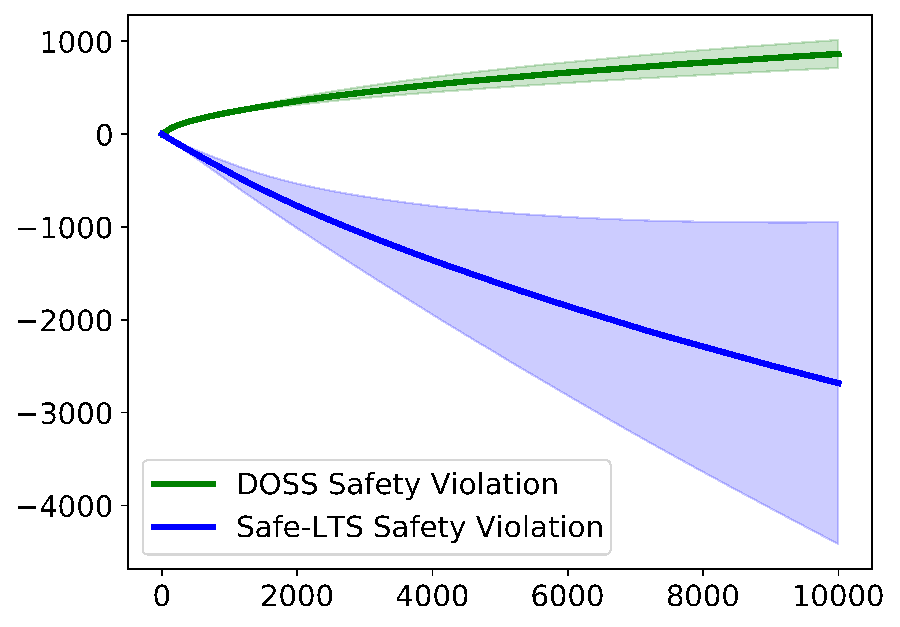}
\caption{\footnotesize Comparing the behaviour of \algoname and safe-LTS on the instance of Example~\ref{example:triangle}. The left plot shows the \emph{raw} efficacy regret, while the right plot is the \emph{raw} safety violations of the two methods, and each reports means and one-standard deviation confidence regions over 30 runs.. Observe that the efficacy performance of safe-LTS is extremely poor, indicating that the algorithm is far from the boundary of the safe set $\safetyset$ for most of its runs. In contrast, the violation properties of \algoname are well-controlled, and almost four times smaller than the efficacy regret of safe-LTS.}
\label{fig:dovslts}
\end{figure}

Figure~\ref{fig:dovslts} compares the behaviour of the \emph{raw efficacy regret} $\sum \ip{\theta, x^* -x_t}$ (left) and the \emph{raw safety violation} $\sum \max_i (\ip{a^i, x_t} - \alpha^i)$ (right) of Safe-LTS and \algoname (since the efficacy regret of \algoname, and the safety-violations of safe-LTS are both essentially $0$, the raw behaviour elucidates more insight). As expected, safe-LTS suffers from $0$ safety regret, since it plays in a pessimistic set. However, this is accompanied by a large efficacy regret, with the mean of over $7000$ at the horizon $T = 10^4$. This arises due to the extreme conservatism of this method, which is evident from its safety violation property: the method has a strong negative (and decreasing still) violation, indicating that it continues to play deep in the interior of the domain for large $T$. Indeed, since over the domain, $\ip{a,x} -\alpha \in [-0.5, 0.5],$ and since the violation at $T = 10^4$ is roughly $-3000,$ this indicates that with a nontrivial probability, the method remains at least $0.25$-separated from the boundary of the safe set.

In comparison, observe that the raw efficacy regret of \algoname is negative, but not nearly as far as the violations of safe-LTS. This indicates that the method is shrinking towards the boundary of the safe set at a much better rate. Of course, this property is similarly illustrated by the violation behaviour: this nearly four times smaller than the efficacy regret of safe-LTS, and concentrates strongly to $\approx 800$ at $T = 10^4$.

\section{Discussion}

The SLB problem is inherently challenging due to the roundwise enforcement of constraints. Our works offers new, and refined insights into both the hardness of the problem through our instance-dependent superlogarithmic lower bound, and to the effectiveness of doubly-optimistic methods for the same through our strong control on $\eff_T.$ In the process, we developed a new dual viewpoint of the SLB problem, by developing gaps for \emph{sets} of constraints, which we believe is a conceptually important tool for such problems. Of course, a number of interesting questions remain open, e.g., are there computationally efficient ways to implement doubly-optimistic strategies for large $U$; or if one can design methods that attain the strong safety guarantees of PO methods, but without making the strong assumptions of prior knowledge of safe points. We believe that tackling these challenges is key to the effective use of bandit feedback in practical scenarios.

\paragraph{Acknowledgements.} {We acknowledge support by the Air Force Research Laboratory grant FA8650-22-C1039, Army Research Office grant W911NF2110246, and the National Science Foundation grants CCF-2007350 and CCF-1955981.}

\bibliographystyle{plainnat}
\bibpunct{(}{)}{;}{a}{,}{,}
\bibliography{Safe_Linear_Bandits}

\clearpage
\appendix

\newpage

\section{Related Work on Pure Exploration.}\label{appx:related_work}

While we study the regret formulation, work on constrained bandits has naturally also appeared in the pure exploration setting. The typical such paper aims to recover arms that are both nearly-safe and nearly-optimal, in a PAC sense. \cite{katz2019top} study this quesiton for finite-armed bandits, and \cite{wang2022best} extend this study under a structured multi-armed bandit setting where each arm has a continuous parameter that must be selected, and monotonically affects reward and safety of the arm. Most pertinently, \cite{camilleri2022active, carlsson2023pure} study best feasible arm identifaction in the linear bandit setting with the same structure as us, although they assume that the set of possible actions is finite and known a priori. It is interesting to note that even in the identification setting, where safety is not enforced during learning, methods that can identify good arms quickly can only give guarantees of safety up to a given precision. This complements our observations in the regret setting.

\section{On the Assumptions, and Background on Online Linear Regression}\label{pre}

We give an expanded discussion of the standard assumptions made in \S\ref{sec:problem_setup}, and discuss a standard result from online linear regression controlling $\sum \|x_t\|_{V_{t-1}^{-1}}$ that is key to our analysis.

\subsection{A closer look at the assumptions}\label{appx:standard_assumptions}

The assumptions made in the main text are slightly simplified version of standard assumptions from the literature on linear bandits. 

\noindent \emph{Boundedness.} The boundedness assumption has two parts: firstly that the underlying parameters are bounded, i.e., $\|\theta\|,\|a^i\| \le 1$ and secondly we assume that the domain is bounded, i.e., $\|x\| \le 1$ for all $x \in \mathcal{X} = \{ Bx \le \beta\}$. 

The bounded domain assumption is used chiefly to ensure that the underlying optimisation problem of interest has finite value. Quantitatively, this may be replaced with a generic bound $\|x\| \le L$ instead without appreciably changing the study. The principal way this affects \algoname is via the choice of the regulariser: instead of setting $V_t = (I + \sum x_s x_s^\top),$ this requires us to set $V_t = \lambda I + \sum x_s x_s^\top$ for some $\lambda > L^2$. Concretely, the validity of of appropriate modification of Lemma~\ref{lem:noise_scale} to handle general regularisation requires using \[ \sqrt{\omega_t(\delta;\lambda)}  =  \sqrt{\frac{1}{2}\log\left( \frac{ (U+1) \det(V_t)^{1/2}\det(\lambda I)^{-1/2}}{\delta}\right)}+\lambda^{1/2} \] for a $\lambda$ that $\lambda \ge \max_t \|x_t\|^2,$ which may be ensured by setting $\lambda \ge L^2$. The main paper simplifies this notational clutter by just setting $\lambda = 1$ and assuming $\|x\| \le 1$. A second aspect that is affected by the quantity $L$ is that the upper bound of Lemma \ref{lemma:norm_of_x_t_squared} would read $\log(1 + TL^2/\lambda d)$ instead of $\log(1 + T/ d)$, which mildly affects some logarithmic terms in the regret bounds (and in fact no bound reported in the main text needs modification if we set $\lambda \ge L^2$ and assume that $L \le d$). 

The assumption of bounded parameters is largely without loss of generality - indeed, if we had a bound $\|\theta\|, \max_i \|a^i\| \le S$ instead, the only change required is that the confidence set radius $\omega_t$ would need to be set as \[ \sqrt{\omega_t(\delta; \lambda, S)} = \sqrt{\omega_t(\delta; \lambda)} + (S-1) \sqrt{\lambda},  \] i.e., only the additive $\sqrt{\lambda}$ term in $\sqrt{\omega_t}$ above would need adjustment. We note that in general, the norm bounds on the various $a^i$ and $\theta$ need not agree, and it is in fact possible to adapt to their norms without prior knowledge of the same, by setting distinct $\omega_t^i$s for each $a^i$, and using the techniques of the recent work of \citet{gales2022norm}.

\noindent \emph{SubGaussianity.} While the subGaussianity condition can also be relaxed (for instance, linear bandits with heavy tailed noise have been studied \citep{shao2018almost}), it yields significant technical convenience whilst remaining quite a generic setting. In the assumption, we concretely assume that the noise is conditionally $1$-subGaussian. This may be relaxed to conditionally $R$-subGaussian. This too can be handled with a small change in $\omega_t$ to \[ \sqrt{\omega_t(\delta;\lambda, R)} =  R\sqrt{\frac{1}{2}\log\left( \frac{ (U+1) \det(V_t)^{1/2}\det(\lambda I)^{-1/2}}{\delta}\right)}+\lambda^{1/2}.\] This change is somewhat stronger than the corresponding change induced by altering $\|\theta\|$ and $\|a^i\|$, since the scaling is now applied to the first term of $\omega_t,$ which grows with $t$ unlike the constant $\sqrt{\lambda}$ penalty.

\noindent \emph{Overall Confidence Radius with General Parameters.} To sum up, under the generic conditions $\|x\|\le L, \|\theta\| \le S, \|a^i\|\le S,$ and $R$-subGaussianity of $\{\gamma_t^i\}$, the entirety of our following analysis will go through, but with the blown up confidence radii \[\sqrt{\omega_t(\delta;\lambda, L, S,R)} = R\sqrt{\frac{1}{2}\log\left( \frac{ (U+1) \det(V_t)^{1/2}\det(\lambda I)^{-1/2}}{\delta}\right)} + S\lambda^{1/2},\] and under the condition $\lambda \ge L^2$. This results in roughly an increase in the regret bounds of a factor of at most $\max(R,S),$ along with a potential increase in the logarithmic terms to $\log(1 + TL^2/\delta)$ instead of $\log(1 + T/\delta)$. For the remainder of our analysis, we shall stick to the default parameters $R= S = L = \lambda = 1$. 

\subsection{Quantitative Bounds from the Theory of Online Linear Regression}\label{appx:online_linear_reg_background}

We conclude the preliminaries with the following generic statement, which holds due to a couple of applications of the matrix-determinant lemma. The result is standard - see the discussions of \citet[Lemma 11]{abbasi2011improved} for historical discussions.

\begin{lemma}
\label{lemma:norm_of_x_t_squared}
Let $\{x_t\}$ be the actions of \algonamenospace. Suppose that for all $t$, $\|x_t\|\le 1$, and let $\lambda\ge 1$. Then for any $T$, 
\[ \sum_{t=1}^T\|x_t\|_{V_{t-1}^{-1}}^2\le \frac{3}{2}\log\left(\frac{\mathrm{det}(V_T)}{\mathrm{det}(\lambda I)}\right) \le \frac{3}{2} d \log\left( 1 +  \frac{T}{\lambda d}\right).\]
\end{lemma}
\begin{proof}[Proof of Lemma~\ref{lemma:norm_of_x_t_squared}]
First notice that since $V_t = V_{t-1} + x_tx_t^\top,$ by the matrix-determinant lemma, \[ \det(V_t) = \det(V_{t-1}) \det(I + V_{t-1}^{-1/2} x_t x_t^T(V_{t-1}^{-1/2})^\top = \det(V_{t-1}) (1 + \|x_t\|_{V_{t-1}^{-1}}^2),\] and induction yields \[ \det(V_T) = \det(\lambda I) \prod (1 + \|x_t\|_{V_{t-1}^{-1}}^2). \] where we have used that $V_0 = \lambda I.$

Now, notice that since $V_{t-1} \succ \lambda I$ for each $t$, it follows that $\|x_t\|_{V_{t-1}^{-1}}^2 \le \|x_t\|^2/\lambda \le 1.$ But for $z \in [0,1], z \le \frac{3}{2} \log(1 + z),$ which implies that \[ \sum \|x_t\|_{V_{t-1}^{-1}}^2 \le \frac{3}{2} \sum \log (1 + \|x_t\|_{V_{t-1}^{-1}}^2) = \frac{3}{2} \log \frac{\det(V_T)}{\det(\lambda I)}. \]

Finally, note that since $V_T$ is positive definite, by an application of the AM-GM inequality, ${\det(V_T)} \le (\mathrm{trace}(V_T)/d)^d$, and further, $\mathrm{trace}(V_T) = d\lambda + \sum_t \|x_t\|_2^2 \le d\lambda + T.$ Further observing that $\det(\lambda I) = \lambda^d,$ we conclude that \[ \log \frac{\det(V_T)}{\det(V)} \le d \log \frac{(d\lambda + T)/d}{\lambda} = d \log \left( 1 + \frac{T}{d\lambda} \right). \]  
\end{proof}

An immediate consequence of the above is the following pair of observations which we shall use frequently. 
\begin{lemma}\label{lem:rho_t_sums}
    Let $\{x_t\}$ be the actions of \algoname run with the parameters $\lambda, \delta$. For every $T > 0,$ \begin{align}
        \label{ineq:sum_rho_t_squared} \sum_{t \le T} \rho_t(x_t;\delta)^2 &\le 3 d^2 \log^2\left(1 + \frac{T}{\lambda d}\right) + 6d\log \left(1 + \frac{T}{\lambda d}\right) \left( \log \frac{U+1}{\delta} + 2\lambda\right),\\
        \label{ineq:sum_rho_t} \sum_{t \le T} \rho_t(x_t;\delta) &\le d\sqrt{3T} \log\left( 1 + \frac{\log T}{d\lambda} \right) + \sqrt{3d T \log\left( 1 + \frac{T}{\lambda d}\right)} \left(\sqrt{2\lambda} + \sqrt{\log \frac{U+1}{\delta}}\right).      
    \end{align}
\end{lemma}

These bounds supply the core bounds needed to convert the control we develop on $\rho_t$ in \S\ref{section:control_on_play_of_bad_BISs} and \S\ref{section:regret} into control on $\eff_T$ and $\saf_T$. Observe that the main terms in the above results do not show dependence on the failure probability parameter $\delta$.

\begin{proof}[Proof of Lemma~\ref{lem:rho_t_sums}]

    Recall that $\rho_t(x_t;\delta) = 2\sqrt{\omega_t(\delta)} \cdot \|x_t\|_{V_{t-1}^{-1}}.$ Further observe that $\omega_t$ is an increasing function of $t$. Immediately by Lemma~\ref{lemma:norm_of_x_t_squared}, \[ \sum \rho_t^2 \le 4\omega_T(\delta) \sum \|x_t\|_{V_{t-1}^{-1}}^2 \le 6d \omega_T(\delta) \log\left( 1 + \frac{T}{d\lambda}\right). \]

    Further, once again applying Lemma~\ref{lemma:norm_of_x_t_squared}, and noting that $(\sqrt{u} + \sqrt{v})^2 \le 2u + 2v,$ \begin{align*} \sqrt{\omega_T(\delta)} &= \sqrt{\lambda} + \sqrt{\frac{1}{2} \log \frac{(U+1)}{\delta} + \frac{1}{4} \log \frac{\det(V_T)}{\det(\lambda I)}} \\ \implies \omega_T(\delta) &\le 2\lambda + \log \frac{(U+1)}{\delta} + \frac{d}{2} \log \left(1 + \frac{T}{\lambda d}\right).\end{align*} 
    
    Multiplying these two bounds controls $\sum \rho_t^2.$

    Further, by the Cauchy-Schwarz inequality, \[ \sum_{t= 1}^T \rho_t \le \sqrt{T} \cdot \sqrt{\sum_{t = 1}^T \rho_t^2}.\] The bound (\ref{ineq:sum_rho_t}) follows upon applying the bound on $\sum \rho_t^2$ above, and then using the trivial relation $\sqrt{u +v} \le \sqrt{u} + \sqrt{v}$.
\end{proof}

Finally, let us argue that the quantity $\rho_t(x_t;\delta)$  indeed controls the noise scale of the problem by showing Lemma~\ref{lem:noise_scale}.

\begin{proof}[Proof of Lemma~\ref{lem:noise_scale}]

We refer to the proof of \citet[Thm.~2]{abbasi2011improved} for the consistency, and observe only that the factor $(U+1)/\delta$ enters our confidence radius $\sqrt{\omega_t(d)}$ by hitting their analysis with the union bound to ensure concentration over the unknown objective and over the $U$ unknown constraints simultaneously. Of course, the factor of $\unk$ arises in the definition of $\matconf$ since each known constraint is already `estimated' exactly by setting $\hat{a}^i_t = a^i$ for $i \in [U+1:U+K]$.

To show that the noise scale limits the deviations $\ip{\ttheta - \theta, x}$, observe that under the assumption of consistency, $\theta \in \confset_t^\theta$. Therefore \[ |\langle \tilde\theta - \theta, x\rangle| \le |\langle \tilde\theta - \hat\theta, x\rangle| + |\langle \theta -\hat\theta, x\rangle|.\] 
By exploiting the positive definiteness of $V_{t-1}$ and the Cauchy-Schwarz inequality, we can further observe that \begin{align*}
    |\langle \theta - \hat\theta, x\rangle| &= |\langle (\theta -\hat\theta)V_{t-1}^{1/2}, V_{t-1}^{-1/2} x\rangle| \le \| \theta - \hat\theta\|_{V_{t-1}} \cdot \|x\|_{V_{t-1}^{-1}}.
\end{align*}
Running the same calculation of $\tilde\theta$ and adding the bounds, we conclude that \[ |\langle \tilde\theta - \theta, x\rangle| \le ( \| \theta- \hat\theta\|_{V_{t-1}} +  \| \tilde\theta - \hat\theta\|_{V_{t-1}} ) \|x\|_{V_{t-1}^{-1}}\]

But both $\theta, \tilde\theta \in \confset_t^\theta,$ which by definitions means that their $V_{t-1}$-norm distance from $\hat\theta$ is bounded by $\sqrt{\omega_{t}(\delta)}.$ The claim is immediate upon recalling that $\rho_t(x;\delta) := 2\sqrt{\omega_{t}(\delta)}\|x\|_{V_{t-1}^{-1}}$.

Of course, the same argument applies to every $\ta^i$, and thus to $\tA$. Again, for known constraints, the radius of the confidence set is $0$, so $\ta^i =\hat{a}^i = a^i,$ and hence the factor of $\unk$ in $|(\tA - A)x| \le \rho_t(x;\delta) \unk$.

Finally, the bounds on $\sum \rho_t(x_t;\delta)$ and $\sum \rho_t(x_t;\delta)^2$ follow directly from Lemma~\ref{lem:rho_t_sums}.
\end{proof}

\section{Appendix on the Structural Behaviour of \algonamenospace}\label{appx:always_a_noisy_index_set} 

This section is devoted to showing the key structural properties of the behaviour of \algoname that we discussed in \S\ref{sec:structural_behaviour}. In particular, we show the main result of \S\ref{sec:BISs}, namely that any point that \algoname plays must noisily activate some BIS. To this end, we first characterise the behaviour of \algoname relative to polytopes contained in the permissible set. Before stating the same, recall that an extreme point of a polytope (and indeed a closed convex set), is any point that is not contained on a line joining two other points in the polytope. Further, each extreme point of a polytope in $\mathbb{R}^d$ must satisfy at least $d$ constraints with equality. For a polytope $\mathcal{P},$ we will denote its extreme points as $\mathcal{E}_{\mathcal{P}}$. 

\begin{lemma}\label{lem:DOSLB_plays_on_polytope_boundaries}
    Suppose that $\mathcal{P}$ is a polytope such that $\mathcal{P} \subset \permset_t$. If \algoname plays in $\mathcal{P},$ then $x_t$ must be an extreme point of $\mathcal{P},$ i.e., $x_t \in \mathcal{P} \implies x_t \in \mathcal{E}_{\mathcal{P}}$. 
\end{lemma}

Let us first argue that the Proposition~\ref{prop:noisy_association} follows from the above Lemma.

\begin{proof}[Proof of Proposition~\ref{prop:noisy_association}]

    For a choice of $\tamat \in \matconf_t,$ define the polytope \[ \mathcal{P}(\tamat) = \{ x : \tamat x \le \alpha\}. \]
    
    Now, observe that \[ \permset_t = \bigcup_{ \{\tamat \in \matconf_t\}} \{ x : \tamat x \le \alpha\} = \bigcup_{ \tamat \in \matconf_t } \mathcal{P}( \tamat),\] i.e., $\permset_t$ can be decomposed as a union of polytopes. But then the selected point $x_t$ must lie in one of these polytopes, say $\mathcal{P}^*$. 
    
    Now, we have that $\mathcal{P}^* \subset \permset_t,$ and $x_t \in \mathcal{P}^*,$ and so by Lemma~\ref{lem:DOSLB_plays_on_polytope_boundaries}, $x_t$ must be an extreme point of $\mathcal{P}^*.$ But this implies that there are at least $d$ linearly independent constraints amidst the $\tamat x\le \alpha$ that $x_t$ activates, i.e., that there exists some $I \subset [1:m]$ such that $|I| = d$ and $\tA(I) x = \alpha(I)$. By definition, then $x_t \in \widetilde{X}^I_t,$ showing the claim.
\end{proof}

It remains to show the preceding Lemma. Before proceeding, let us comment that the statement above is intuitively obvious, but appears to be somewhat cumbersome to prove (as the argument below suggests, although nothing says that a cleaner proof could not be found). Of course, this statement extends also to the OFUL algorithm, and to our knowledge this has not been directly argued previously: instead, when working on polytopal domains, typically it is directly stated that it suffices to play on the extreme points of the polytope.

\begin{proof}[Proof of Lemma~\ref{lem:DOSLB_plays_on_polytope_boundaries}]

    Suppose that $x_t \in \mathcal{P}$. Then, due to the optimistic choice, there also exists some $\ttheta_t \in \confset_{t}^0$ such that \[ (\ttheta_t, x_t) \in \argmax_{\ttheta \in \confset_{t}^0, x \in \mathcal{P}} \ip{\ttheta, x}. \] Notice also that $x_t$ is a solution of the linear program $\max_{x \in \mathcal{P}} \ip{\ttheta_t, x}$, and so lies on the boundary of $\mathcal{P}$. Similarly, $\ttheta_t$ lies on the boundary of $\confset_{t}^0.$ We need to argue that $x_t$ must in fact be an extreme point of $\mathcal{P}$, i.e., it does not lie in the interior of some face of dimension $\ge 1$ of $\mathcal{P}.$
    
    For this, first suppose for the sake of contradiction that $x_t$ lies in the interior of some 1-dimensional face of $\mathcal{P}$, say $\mathcal{F}$. Let $u$ be the direction of variation of $\mathcal{F}$. Then it must hold that $\ip{\ttheta_t, u} = 0,$ else $\ip{\ttheta_t, x_t + \varepsilon u}$ would exceed $\ip{\ttheta_t, x_t}$ for some small choice of $\varepsilon$. Now, let us rotate the domain so that $u$ is directed along one coordinate axis, and project onto the 2D subspace spanned by the (orthogonal) directions $u$ and $\ttheta_t$. Next, rescale the vectors so that both $u$ and $\ttheta_t$ have norm $1$, and finally translate the polytope so that the $u$th component of $x_t$ is $0$. Notice that the projection of an ellipsoid is an ellipsoid, and so doing the same transformations to $\confset_{t-1}^\theta$ produces a 2-dimensional convex confidence ellipsoid $D$.
    
    Let us relabel the axes of the resulting system as $u_1$ and $u_2.$ In the resulting coordinate system, $\ttheta = (0,1),$ and $\mathcal{F}$ is a line segment of the form $\{ u_1 \in [p,q], u_2 = r \},$ where $p < 0 < q, r = \ip{\ttheta_t, x_t}/\|\ttheta_t\|$ and $x_t = (0,r)$. Observe that $\ttheta_t$ must lie on the boundary of $D$. We shall argue that there is some other $z \in \mathcal{F}$ and some other $\phi \in D$ such that $\ip{z, \phi} > r,$ which violates the assumption.
    
    We first take the case of $r > 0.$ Observe that if any point of $D$ has $u_2$ coordinate greater than $1$, then we immediately have a contradiction, since then for such a point $\phi, \ip{\phi, x_t} > \ip{\ttheta_t, x_t}$. But, since $\ttheta_t = (0,1) \in D,$ it follows that the ellipse $D$ is tangent to $u_2 = 1$. But this means that for small $\varepsilon,$ $D$ must contain points $\phi_\varepsilon = (\varepsilon, 1 - f(\varepsilon))$ where $0 \le f(\varepsilon) = O(\varepsilon^2)$. But this implies a contradiction - indeed, take $\varepsilon > 0,$ and consider $z_\varepsilon = (\varepsilon^{1/2},r).$ Then $z_\varepsilon \in \mathcal{F}$ for small enough $\varepsilon,$ and $\ip{z_\varepsilon, \phi_\varepsilon} -r = \varepsilon^{3/2} - r f(\varepsilon).$ Since $f(\varepsilon) = O(\varepsilon^2),$ this is positive for small enough $\varepsilon,$ demonstrating a contradiction.
    
    If $r < 0,$ the same argument can be run mutatis mutandis - now $D$ must lie above the line $u_2 = 1,$ but still be tangent to it, and we can develop points of the form $(\varepsilon, 1 + f(\varepsilon))$ for $0 \le f = O(\varepsilon^2)$ in $D$, and the analogous inner product $\ip{z_\varepsilon, \phi_\varepsilon} - r = \varepsilon + rf(\varepsilon)$ which is again positive for small enough $\varepsilon$.
    
    Finally, we have the case $r = 0,$ wherein $x_t$ lies at the origin. But in this case any point in $D$ of non-zero $u_1$ coordinate serves as a contradiction (since either $(p,0)$ or $(0,q)$ will yield a positive inner product). 
    
    Together, the above paragraphs imply that $x_t$ cannot lie on the interior of an edge of $\mathcal{P}$. But this argument generalises to the interior of any non-trivial face. Indeed, since $\ttheta_t$ must be orthogonal to the affine subspace formed by this face, we can argue that there must be a point in the interior of a 1-D face (that forms a boundary of the larger face) that must also attain the optimal value for $\ip{\ttheta, x}$, and then run the above argument for this point. It follows that $x_t$ cannot lie in the interior of any non-trivial face of $\mathcal{P}$.
\end{proof}

The above argument is not restricted to confidence ellipsoids of the form of \S\ref{sec:confidence_sets}, but extends to any $\confset_t$ with a smooth and convex boundary. Indeed, this further extends to convex $\confset_t$ with continuous boundaries, barring the case where $\ttheta_t$ is itself the extreme point of a polytope (with large `curvature' at $\ttheta_t$). In such a case the property that $f(\varepsilon) = O(\varepsilon^2)$ does not hold, and a more global argument may be needed. One attack may pass through the use of continuous noise, in which case the confidence sets would almost surely not produce any extreme points that are orthogonal to the faces of a polytope (since such directions lie in a union of a finite number of dimension $d-1$ affine subspaces, which in turn is Lebesgue null), and so we may almost surely avoid this disadvantageous case.

Let us also note the following interesting observation that can also be inferred using Lemma~\ref{lem:DOSLB_plays_on_polytope_boundaries}, and further characterises the behaviour of doubly-optimistic play. 

\begin{proposition}\label{prop:appx_safe_bound_opt_bis}
Suppose that all confidence sets are valid. Then there exists at least one BIS $I$ that $x_t$ noisily activates, and such that $A(I) x_t \ge \alpha(I)$. 
\end{proposition}

In other words, for at least one BIS, the action $x_t$ not only noisily activates it, it further either activates it or violates all of the true constraints of this BIS. Notice that if the BIS shown to exist above has at least one unknown constraint, then this basically means that \algoname must violate safety (since meeting this with equality for the unknown constraint would be rare). 

\begin{proof}[Proof of Proposition~\ref{prop:appx_safe_bound_opt_bis}]

Fix $x_t$. We call $\tamat \in\matconf_t$ a witness for $x_t$ if $\tamat x_t \le \alpha,$ i.e., if $\tamat$ witnesses the presence of $x_t$ in $\permset_t$. Since $x_t$ is the optimistic optimum over the entirety of $\permset_t,$ it follows that for every witness $\tamat$ of $x_t,$ it holds that $x_t \in \argmax_x \max_{\tilde\theta\in \confset_t^\theta} \langle \tilde\theta, x\rangle : \tamat x \le \alpha.$ 

Now, let $I_0$ be all of the constraints that $x_t$ noisily activates, and let $I_{\ge} :=\{i \in [1:m] : \ip{a^i,x_t} \ge \alpha^i\}$. We claim that $|I_{\ge}| \ge d,$ which suffices to show the claim.

For the sake of contradiction, assume that $|I_{\ge}| \le d-1$. For each $i \in I_0\setminus I_{\ge}$, we have $\ip{a^i,x_t} < \alpha^i$. Let us form the matrix $\tamat_{<}$ formed by taking each of the $i$th rows in $\tamat$ for which $i \in I_0\setminus I_{\ge},$ and replacing the $\ta^i$ in the row by $\ta_{<}^i = a^i$. This matrix remains a witness, since the resulting $\tamat_<$ lies in $\matconf_t$ (as we have replaced rows by the rows of $A$, each of which lie in the corresponding confidence sets for individual rows), and by definition for each replaced row, $\ip{\ta_<^i,x_t} < \alpha^i$, since each such $i$ lies in $I_0 \setminus I_{\ge}$. 

Then $x_t$ lies in the interior of the polytope $\mathcal{P}_< := \{x : \tamat_< x \le \alpha\},$ since by construction it activates at most $|I_{\ge}| \le d-1$ constraints of this matrix. But since $\tamat_< \in \matconf_t,$ it holds that $\mathcal{P}_< \subset \permset_t$, and thus the algorithm plays in the intrior of a polytope contained in the permissible set, contradicting Lemma~\ref{lem:DOSLB_plays_on_polytope_boundaries}. Therefore, our hypothesis is untenable, and $|I_{\ge}| \ge d$.
\end{proof}

\section{Controlling the Play of Suboptimal BISs}

We now show the noise scale lower bound, and the subsequent control on the play of suboptimal BISs as discussed in \S\ref{sec:structural_behaviour}.

\subsection{Localising Actions when a BIS is Activated}\label{appx:local_box}

We show Lemma~\ref{lemma:local_box} as a simple consequence of consistency and optimism.

\begin{proof}[Proof of Lemma~\ref{lemma:local_box}]

    Suppose that the confidence sets are consistent, and that $x_t$ noisily activates the BIS $I$. Since $x_t$ is the action of \algonamenospace, it is also permissible. Together, these two properties imply that there exists some $\tamat \in \matconf_t$ such that \begin{align*} \tamat x_t &\le \alpha\\ \tamat(I) x_t &= \alpha(I) \end{align*}

    But, since $\matconf_t$ is consistent, Lemma~\ref{lem:noise_scale} yields \[ Ax_t - \rho_t\unk \le \tamat x_t \le A x_t + \rho_t \unk.\] The claim follows directly from this, since \[\alpha \ge \tamat x_t \ge A x_t - \rho_t\unk  \implies Ax_t \le \alpha + \rho_t \unk,\] and \[\alpha(I) =  \tamat(I)x \le A(I) x_t + \rho_t\unk(I) \implies A(I) x_t \ge \alpha(I) - \rho_t\unk(I).\]

    Further, due to the optimistic selection of $x_t$, it is a maximiser amongst the permissible set of $\max_{\tilde\theta \in \confset_t^\theta} \langle \tilde\theta, x\rangle$. But under consistency, $\theta \in \confset_t^\theta,$ and $x^* \in \permset_t$. Thus, it follows that if $\ttheta$ is the optimal choice in the above program, then \[ \langle \ttheta, x_t\rangle \ge \ip{\theta, x^*}.\] But, again using consistency and Lemma~\ref{lem:noise_scale}, it holds that $\langle\ttheta, x_t\rangle \le \langle \theta, x_t\rangle + \rho_t,$ from which the claim is forthcoming.   
\end{proof}

\subsection{Proof of Noise Scale Lower Bound and the Positivity of the Gaps of Suboptimal BISs}\label{appx:gaps_make_sense_proofs}

The argument underlying the proof of the noise-scale lower bound is essentially encapsulated in \S\ref{sec:intuitive_gap}, but refined through the use of the LP $P(\zeta;I)$. The bulk of the following proofgoes into showing that the gap we define is meaningful, i.e., that if $I$ is a suboptimal BIS, then $\max(\zeta_*(I), \eta_*(I)) > 0$. This essentially boils down to showing that $\spread(I)$ is finite for feasible BISs.

\begin{proof}[Proof of Lemma~\ref{lem:spread_is_finite_AND_noise_scale_lower_bound}]

    We will first show that under consistency of the confidence sets, playing a suboptimal BIS $I$ implies that $\rho_t(x_t;\delta) \ge \max(\zeta_*(I), \eta_*(I))$. Observe that under the assumption of consistency, \[ \ip{\theta, x_t} \le P(\rho_t;I),\] since $x_t \in \mathcal{T}(\rho_t;I)$ by Lemma~\ref{lemma:local_box}. Further, by the final line of Lemma~\ref{lemma:local_box}, \( \ip{\theta, x_t} \ge \ip{\theta, x^*} - \rho_t.\)

    Since $x_t \in \mathcal{T}(\rho_t;I),$ this set is nonempty, and therefore by definition $\rho_t \ge \zeta_*(I)$. Note that if $\zeta_*(I) = \infty,$ we can conclude already, since this means that $\rho_t(x_t;\delta) \ge \infty \ge \max(\zeta_*(I),\eta_*(I))$. If $\zeta_*(I) < \infty,$ then by the definition of the spread $\spread(I)$, and the efficacy separation $\gamma(I)$, we have \[ P(\rho_t;I) \le P(\zeta_*(I);I) + \spread(I) (\rho_t - \zeta_*(I)) = \ip{\theta, x^*} -\gamma(I) + \spread(I) (\rho_t - \zeta_*(I)). \]

    But then we conclude that \[ -\rho_t \le -\gamma(I) + \spread(I)(\rho_t - \zeta^*(I)) \iff \rho_t(\spread(I) + 1) \ge \gamma(I) + \spread(I) \zeta_*(I) \iff \rho_t \ge \eta_*(I). \] Thus, the claim follows.

    We now proceed to argue that for any suboptimal BIS $I$, at least one of $\zeta_*(I)$ and $\eta_*(I)$ is positive. Fix the BIS $I$. Note that if $\zeta_*(I) = \infty,$ then there is nothing to show. So, suppose $\zeta_*(I) < \infty$. By expanding out the definition of $\mathcal{T}(\zeta;I),$ the program $P$ is  \begin{align*} P(\zeta;I) = \max_x &\ip{\theta,x} \\ \textrm{s.t. } &A x \le \alpha + \zeta\unk\\ -&A(I)x \le -\alpha(I) + \zeta\unk(I). \end{align*}
    We recall that this is a linear program, which is of course evident in the above. Since $\zeta_*(I) < \infty,$ the above program is feasible for $\zeta \ge \zeta_*(I)$. Further, since $\mathcal{X} \supset \mathcal{T}(\zeta;I)$ is a bounded polytope, the program is finite. Thus strong duality applies to the above program.

    Let us introduce dual variables $(\lambda, \mu)$ respectively for the two blocks of constraints. By standard techniques, the dual program is \begin{align*} D(\zeta;I) = \min_{\lambda, \mu} &\ip{\lambda, \alpha + \zeta \unk} + \ip{\mu, -\alpha(I) + \zeta\unk(I)} \\
    \textrm{s.t. } &A^\top \lambda - A(I)^\top \mu = \theta,\\
    &\lambda \ge 0, \mu \ge 0.\end{align*}

    For succinctness, let us write \begin{align*} f(\lambda, \mu) &= \ip{\lambda, \unk} + \ip{\mu, \unk(I)} \\ g(\lambda, \mu) &= \ip{\lambda, \alpha + \zeta_*(I)\unk} + \ip{\mu, -\alpha(I) + \zeta_*(I) \unk(I)}, \\ h(\lambda, \mu)  &:=  A^\top \lambda - A(I)^\top \mu - \theta.\end{align*}

    Further, let $\boldsymbol \lambda = (\lambda, \mu)$. We can succinctly write the dual as \[ D(\zeta;I) = \min_{\boldsymbol\lambda} (\zeta - \zeta_*(I)) f(\boldsymbol\lambda) + g(\boldsymbol\lambda) : h(\boldsymbol\lambda) = 0, \lambda\ge 0, \mu \ge 0.\]

    Note that since the primal is bounded and feasible for $\zeta \ge \zeta_*(I),$ so is the dual, and by strong duality $D(\zeta_*(I);I) = P(\zeta_*(I);I).$ But \[ D(\zeta_*(I);I) = \min_{\boldsymbol\lambda} g(\boldsymbol\lambda) : h(\boldsymbol\lambda) = 0, \lambda \ge 0, \mu \ge 0. \] It follows that the set \[ \mathcal{F}:= \{\boldsymbol \lambda:  g(\boldsymbol\lambda) \le P(\zeta_*(I);I), h(\boldsymbol\lambda) = 0, \lambda \ge 0, \mu \ge 0\} \] is nonempty. Observe that $\zeta$ does not appear anywhere in the definition of $\mathcal{F}$.

    Let us define the two programs \begin{align*}
        D'(\zeta;I) &:= \min_{\boldsymbol\lambda} (\zeta - \zeta_*(I)) f(\boldsymbol\lambda) + g(\boldsymbol\lambda) : \boldsymbol\lambda \in \mathcal{F},\\
        E(I) &:= \min_{\boldsymbol\lambda} f(\boldsymbol\lambda) : \boldsymbol\lambda \in \mathcal{F}
    \end{align*}

    Note that both of the above programs are feasible. As a feasible minimisation program we also have that $E(I) < \infty$. Further, since introducing extra constraints cannot decrease the value of a minimisation program, we note that $D(\zeta;I) \le D'(\zeta;I)$. But observe that since the constraints of $D'(\zeta;I)$ include the requirement that $g(\boldsymbol\lambda) \le P(\zeta_*(I);I),$ we have for every $\zeta \ge \zeta_*(I)$ that \begin{align*} D'(\zeta;I) &\le P(\zeta_*(I);I) + \min\{ (\zeta - \zeta_*(I)) f(\boldsymbol\lambda) : \boldsymbol\lambda \in \mathcal{F}\} \\ &= P(\zeta_*(I);I) + (\zeta - \zeta_*(I)) \cdot \min\{f(\boldsymbol\lambda) : \boldsymbol\lambda \in \mathcal{F}\} \\
    &= P(\zeta_*(I);I) + (\zeta - \zeta_*(I)) E(I).\end{align*}

    But, then by strong duality, \[ P(\zeta;I) = D(\zeta;I) \le P(\zeta_*(I);I) + (\zeta - \zeta_*(I)) E(I),\] and we conclude that $\spread(I) \le \max(0,E(I)) < \infty$. 

    Now, since $\spread(I)$ is finite, in order to show that $\max(\zeta_*(I), \eta_*(I)) > 0,$ it suffices to argue that for any suboptimal BIS, $\max(\zeta_*(I), \gamma(I)) > 0.$ But observe that if $\zeta_*(I) = 0,$ then $\lim_{\zeta \searrow 0} P(\zeta;I) > -\infty,$ and due to the right-continuity of $P$, this implies that $P(0;I) > -\infty \implies \mathcal{X}^I \neq \emptyset,$ in other words, $I$ is a feasible BIS. But if a BIS $I$ is both feasible and suboptimal, then for every $x \in I,$ it must hold that $\ip{\theta, x} < \ip{\theta, x^*}$, since otherwise $I$ would be optimal. But, since $\mathcal{X}^I = \mathcal{T}(0;I)$ is a compact set, this means that $P(\zeta_*(I);I) = P(0;I) < \ip{\theta, x^*} \iff \gamma(I) > 0.$ 
\end{proof}

\subsection{Bounding the Play of Suboptimal BISs}\label{appx:suboptimal_BIS_cannot_be_played_too_often}

With the above ingredients in place, we show the main result of \S\ref{section:control_on_play_of_bad_BISs}.

\begin{proof}[Proof of Theorem~\ref{thm:number_of_bad_bfs}]

    Let us again abbreviate $\rho_t(x_t;\delta)$ as $\rho_t$. By Lemma~\ref{lem:spread_is_finite_AND_noise_scale_lower_bound}, if a suboptimal BIS I is played, then $\rho_t \ge \max(\eta_*(I), \zeta_*(I)).$ But then any time a suboptimal BIS is played, $\rho_t \ge \min \{ \max(\eta_*(I), \zeta_*(I)) : I \textrm{ is a suboptimal BISs} \},$ i.e., $\rho_t \ge \Gamma$.

    Now observe that \begin{align*}
        \sum_{t = 1}^T \indi\{  \exists \textrm{suboptimal BIS } I : x_t \in \widetilde{\mathcal{X}}_t^I\} &\le \sum_{t = 1}^T \indi\{\rho_t \ge \Gamma\} \\
        &\le \sum_{t = 1}^T \frac{\rho_t^2}{\Gamma^2} \indi\{\rho_t \ge \Gamma\} \\
        &\le \Gamma^{-2} \sum_{t\le T} \rho_t^2,
    \end{align*}
    where the second inequality is using that if $\rho_t \ge \Gamma,$ then $\rho_t/\Gamma \ge 1.$ Applying Lemma~\ref{lem:rho_t_sums} immediately bounds the above as \( O\left( \frac{d^2 \log^2T + d\log(T) \log(U/\delta)}{\Gamma^2} \right). \)
\end{proof}

\section{Proofs of Bounds on Efficacy Regret and Safety Violations}\label{appx:regret_bounds}

We proceed to discuss the proofs of the results of \S\ref{section:regret}.

\subsection{The Efficacy of the Actions of \algoname when Activating Optimal BISs}\label{appx:optimal_BIS_is_good}

Our first order of business is to argue that playing only optimal BISs leads to actions $x_t$ that are `over-efficient', i.e., satisfy $\ip{\theta, x_t} \ge \ip{\theta, x^*}$. The following basic result is useful in our argument. 

\begin{lemma}\label{lemma:genericity_means_full_rank}
    For a BIS $I$, define $K_I = I \cap [U+1:m]$ to be the indices of the known constraints in $I$. Under the genericity of noise assumption, for any BIS $I$ such that $A(K_I)$ is full row rank, for any $t \ge d,$ it holds almost surely that $\hat{A}_t(I)$ is full rank. 
\end{lemma}
\begin{proof}[Proof of Lemma~\ref{lemma:genericity_means_full_rank}]
    Notice that since for any $i$, the noise in the feedback $S_t^i$ is generic, it does not concentrate in any low-dimensional subspace of $\mathbb{R}^d$. This in turn means that the probability that any $\hat{a}^i_t$ lies in a low-dimensional subspace of $\mathbb{R}^d$ is exactly zero. The claim follows immediately: since $|I\setminus K_I| \le d,$ each $\hat{a}^i_t$ with probability one does not lie in the span of $\{\hat{a}^j\}_{j \in I \setminus \{i\}},$ and since by assumption the $A(K_I)$ is full rank.   
\end{proof}

With this in hand, we argue Lemma~\ref{lemma:optimally_associated_BIS_are_good} by exploiting the weak-nondegeneracy condition of Assumption~\ref{assumption:non_deg}.

\begin{proof}[Proof of Lemma~\ref{lemma:optimally_associated_BIS_are_good}] We need to show that if \emph{all} of the BISs $x_t$ noisily activates are optimal, then $\ip{\theta, x_t} \ge \ip{\theta, x^*}$, which comprises the bulk of this proof. To this end, let us fix one such BIS, $I$. 

By Assumption~\ref{assumption:non_deg}, we know that $\{x^*\} = \mathcal{X}^I$, and that $I$ is full-rank. Notice that as a result, we may write \[\ip{\theta, x^*} = \max \ip{\theta,x} : \AI x = \ali. \] Indeed, due to the fact that $I$ is full rank, the latter equality constraints already enforce that $x^*$ is the sole feasible point.  Further, by strong duality, there exists a choice of vectors $\mu$ such that \[ \mu^\top \AI = \theta^\top. \]

Due to the optimistic selection rule, and the fact that $x_t$ noisily saturates $I$, it must hold that $x_t$ is a solution to \[ \max_{ \tilde\theta \in \confset_t^\theta, \tA \in \matconf_t} \max_x \langle \tilde\theta,x\rangle : \tai x = \ali, \tA x \le \alpha,\] where the maximisation over $\tA$ is equivalent to optimisitic selection over $\permset_t = \{x : \exists \tA: \tA x \le \alpha\}$, and the equality constraint arises since $x_t$ noisily activates $I$. Now observe that in the optimisation above, we may restrict attention to $\tA$ such that $\tai$ is full rank. Indeed, if this optimal choice were rank-deficient, then since the feasible set remains a polytope, there must exist some other constraints amongst the $\tA$ besides those in $I$ that are activated by $x_t$ (since otherwise we would be playing on the interior of a polytope, and thus violating Lemma~\ref{lem:DOSLB_plays_on_polytope_boundaries}). By dropping some linearly dependent rows, this would yield a different index set $I'$ that $x_t$ activates, and which is not rank-deficient. By the hypothesis, this index set must also be optimal, and we can run the argument for $I'$ instead. But then note that $x_t$ is exactly characterised by the equality conditions imposed by noisily activating the BIS $I$, which means that $x_t$ is the optimiser of 

\[ \max_{\substack{\tilde\theta \in \confset_t^\theta, \tA \in \matconf_t, \\ \tilde{M}(I, \tA) \textrm{ is full-rank} }} \max_x \langle \tilde\theta, x\rangle : \tai x = \ali.\] 

Now, let us write $\tA = A + \delta A, \tilde\theta = \theta + \delta \theta, x = x^* + \delta x.$ Further denote the optima as $\delta\theta_t, \delta A_t, \delta x_t$. With this notation, our goal is to show that $\ip{\theta, \delta x_t} \ge 0.$ To this end, observe that since the program above has the constraint $\tai x = \ali = \AI x^*$ we find that \begin{align*} \tai x = \AI x^* + \delta \AI x + \AI \delta x = \alpha(\indexU) &\iff \AI \delta x = - \delta \AI x,\end{align*} which imply that 
\begin{align}
    \ip{\theta, \delta x} &= \ip{ A^\top \mu, \delta x} = \ip{ \mu, A \delta x} = - \ip{\mu, \delta A x} \notag \\ \iff \ip{\theta, \delta x}  &= \sum_{i \in I} -\mu^i \ip{\delta A^i, x} \label{eqn:theta_times_deltax}
\end{align}

Thus, we can rewrite the program as 
\[ \max_x \max_{\delta\theta, \delta A}  \ip{\theta, x^*} + \ip{\delta\theta, x} - \ip{\mu, \delta \AI x} : \tai x = \ali. \]

Now, recall that the confidence sets are constructed around the RLS estimates $\hat{a}^i_t$ and $\hat\theta_t,$ i.e., \[ \confset_t^\theta = \{\tilde\theta: \{\tilde\theta- \hat\theta_t\|_{V_{t-1}} \le \omega_t\}, \confset_t^i = \{\ta : \|\ta - \hat{a}^i_t\|_{V_{t-1}} \le \omega_t \unk^i\}.\] To clearly express the choice of $\delta\theta, \delta A,$ we define \begin{alignat}{4} &\Delta\theta_t = \hat\theta_t - \theta, &&\Delta a^i_t = \hat{a}^i_t - a^i, &&\Delta A = \hat{A}_t - A \notag \\ 
&\partial\theta = \tilde\theta - \hat\theta_t && \partial a^i = \ta^i - \hat{a}^i_t, &&\partial A = \tA - \hat{A}_t. \notag \end{alignat}

Observe then that \[ \delta\theta = \Delta\theta_t + \partial \theta; \delta a^i = \Delta a^i_t + \partial a^i. \] Further, the decision variables of the program are only the $\partial \theta$ and $\partial a^i$s, which lie in the set $\|\partial \theta \|_{V_{t-1}} \le \omega_t$ and $\|\partial a^i\|_{V_{t-1}} \le \omega_t \unk^i.$ Let us denote $U_I = I \cap [1:U]$ and $K_I = I \cap [U+1:m]$, and observe that $\Delta a^i = \partial a^i = 0$ for $i \in K_I$. Incorporating this structure, we can write the program as 
\begin{align*}
    \ip{\theta, x^*} + \max_x \max_{\partial\theta, \partial A}  &\ip{\Delta\theta_t, x} +  - \sum_{i \in U_I} \mu^i \ip{\Delta a^i_t,x} +\ip{\partial\theta, x} - \sum_{i \in U_I} \mu_i \ip{\partial a^i,x} . \\
    \textrm{s.t. } \quad &\ip{a^i + \Delta a^i_t,x} + \ip{\partial a^i, x} = \alpha^i \quad \forall i \in I, \\
    &\|\partial \theta\|_{V_{t-1}} \le \omega_t\\&\|\partial a^i\|_{V_{t-1}} \le \omega_t \unk^i \quad \forall i \in I.
\end{align*}

But now observe that the optimal choice of $\partial\theta$ in the above is exactly $\omega_t/\|x\|_{V_{t-1}^{-1}} V_{t-1}^{-1} x$. Indeed, recall that $\|u\|_{V_{t-1}} = \sqrt{u^\top V_{t-1} u} = \|V_{t-1}^{1/2} u\|$, and similarly $\|u\|_{V_{t-1}^{-1}} = \|V_{t-1}^{-1/2} u\|$.  By the Cauchy-Schwarz inequality, $\ip{\partial \theta, x} = \ip{V_{t-1}^{1/2} \partial\theta , V_{t-1}^{-1/2}x} \le \|\partial\theta\|_{V_{t-1}} \|x\|_{V_{t-1}^{-1}},$ and this is extremised when $V_{t-1}^{1/2} \partial \theta \propto V_{t-1}^{-1/2} x \iff \partial \theta \propto V_{t-1}^{-1} x.$ Further, if for a scalar $\varphi,$ $\partial \theta = \varphi \cdot V_{t-1}^{-1}x,$ then \[ \partial\theta^\top V_{t-1} \partial\theta = \varphi^2 x^\top V_{t-1}^{-1} V_{t-1} V_{t-1}^{-1} x =\varphi^2 x^\top V_{t-1}^{-1} x,\] or equivalently, $\| \varphi \cdot V_{t-1}^{-1} x\|_{V_{t-1}} = |\varphi| \|x\|_{V_{t-1}^{-1}}$, which means that to obey $\|\partial\theta\|_{V_{t-1}} \le \omega_t,$ we must set $\partial\theta = \pm\frac{\omega_t}{\|x\|_{V_{t-1}^{-1}}} V_{t-1}^{-1}x$, and of these the $+$ solution gives a positive value, and so is optimal.

Further notice that the optimal choice of $\partial a^i$ for $i \in U_I$ must similarly be aligned with $V_{t-1}^{-1}x$. Indeed, write $V_{t-1}^{1/2}\partial a^i = \omega_t \sigma^i V_{t-1}^{-1/2} x + \psi^i,$ where $\sigma^i$ is a scalar, and $\psi^i$ is a vector such that $\ip{\psi^i, V_{t-1}^{-1/2}x} = 0.$ Then observe that due to the orthogonality, \[ \|\partial a^i\|_{V_{t-1}}^2 = \ip{V_{t-1}^{1/2}\partial a^i, V_{t-1}^{1/2}\partial a^i} = \ip{ \omega_t \sigma^i V_{t-1}^{-1/2} x + \psi^i, \omega_t \sigma^i V_{t-1}^{-1/2} x + \psi^i} = \omega_t^2(\sigma^i)^2\|x\|_{V_{t-1}^{-1}}^2 + \|\psi^i\|^2,  \] and so the the constraint on $\partial a^i$ becomes $(\omega_t \sigma^i)^2 \|x\|_{V_{t-1}}^{-1} + \|\psi^i\|^2 \le \omega_t^2.$ But $\psi^i$ affects neither the first constraint on $\ip{a^i + \Delta a^i_t +\partial a^i,x}$, nor the objective, since \[ \ip{\partial a^i,x} = \ip{V_{t-1}^{1/2} \partial a^i, V_{t-1}^{-1/2} x} = \ip{\omega_t \sigma^i V_{t-1}^{-1/2}x, V_{t-1}^{-1/2}x} + \ip{\psi^i, V_{t-1}^{-1/2} x}  = \sigma^i \|x\|_{V_{t-1}^{-1}}^2. \] This means that dumping any energy into $\psi^i$ affects neither the constraints nor the objective, so we can safety set it to zero in the following (in fact, as we shall see below, it must be zero since $\sigma^i$ must saturate). This allows us to considerably simplify the above program: by introducing the real valued variables $\sigma^i$ for $i \in I$, and noting that $\partial a^i = 0$ for $i \in K_I$ can be achieved by demanding $(\sigma^i)^2\|x\|_{V_{t-1}^{-1}} \le 0 = \unk^i$ for $i \in K_I,$ we may rewrite the program above as 
\begin{align*}
    \ip{\theta, x^*} + \max_x \max_{\{\sigma^i\}}  &\ip{\Delta\theta_t, x} - \sum_{i \in U_I} \mu^i \ip{\Delta a^i_t,x} + \omega_t \|x\|_{V_{t-1}^{-1}} - \sum_{i \in U_I} \mu^i \sigma^i \omega_t\|x\|_{V_{t-1}^{-1}}^2. \\
    \textrm{s.t. } \quad &\ip{a^i + \Delta a^i_t,x} = \alpha^i - \omega_t \sigma^i \|x\|_{V_{t-1}^{-1}}^2 \quad \forall i \in I, \\
    &\ip{a^i,x} = \alpha^i \quad \forall i \in K_I \\
    &(\sigma^i)^2 \|x\|_{V_{t-1}^{-1}}^2 \le \unk^i \quad \forall i \in I. 
\end{align*}

Finally, observe that the first pair of constraints can be succinctly written in terms of $\hat{A}_t(\indexU),$ giving us the following restatement, where $\sigma$ is the vector formed by stacking the $\sigma^is$. \begin{align*}
    \ip{\theta, x^*} + \max_x \max_{\sigma}  &\ip{\Delta\theta_t, x} - \sum_{i \in I} \mu^i \ip{\Delta a^i_t,x} + \omega_t \|x\|_{V_{t-1}^{-1}} - \ip{\mu,\sigma} \omega_t\|x\|_{V_{t-1}^{-1}}^2. \\
    \textrm{s.t. } \quad &\hat{A}_t(I) x = \alpha(I) - \omega_t  \|x\|^2_{V_{t-1}^{-1}} \sigma \\
    &(\sigma^i)^2 \|x\|_{V_{t-1}^{-1}}^2 \le \unk^i \quad \forall i \in I.
\end{align*}

But notice that $A(K_I)$ is  full row rank by assumption, and thus applying Lemma~\ref{lemma:genericity_means_full_rank}, with probability one, $\hat{A}_t(I)$ is full-rank. But this means that every value of $\sigma$ that meets the final constraint is feasible for the above program, since we can find an appropriate $x$ by inverting $\hat{A}_t(I)$. Of course, then the optimal choice of $\sigma^i$ is then $-\unk^i\mathrm{sign}(\mu^i) /\|x\|_{V_{t-1}^{-1}},$ telling us that for each $i\in U_i,$ the optimal $\partial a^i$ at time $t$ is  \[ \partial a^i_{t} = -\unk^i \mathrm{sign}(\mu^i)\omega_t  V_{t-1}^{-1}x/\|x\|_{V_{t-1}^{-1}} \implies  \mu^i \ip{\partial a^i_t, x} = \unk^i \omega_t|\mu^i| \cdot \|x\|_{V_{t-1}^{-1}}. \]

Now, finally, we observe that for each $i \in U_I,$ and every $x$, $\omega_t|\mu^i|\|x\|_{V_{t-1}^{-1}} -\mu^i \ip{\Delta a^i_t, x} \ge 0$. Indeed, for $i \in K_I$ this is trivial since both $\Delta a^i_t, \partial a^i_t$ are $0$ for such $i$. For $i \in U_I,$ since the confidence sets are consistent, we know that $a^i \in \confset^i_t \iff \|\Delta a^i_t\|_{V_{t-1}} \le \omega_t.$ But then \[ | \mu^i \ip{\Delta a^i_t, x} | = |\mu^i| |\ip{V_{t-1}^{1/2} \Delta a^i_t, V_{t-1}^{-1/2} x}\| \le |\mu^i| \|\Delta a^i_t\|_{V_{t-1}}  \| x\|_{V_{t-1}^{-1}} \le |\mu^i|\omega_t \|x\|_{V_{t-1}^{-1}}.\] 

But now we are in business. Indeed, using (\ref{eqn:theta_times_deltax}), we finally have  \begin{align*}  \ip{\theta, \delta x_t} &= \sum_{i \in I} -\mu^i \ip{\delta a_t^i, x_t} =  \sum_{i \in I} - \mu^i \ip{\partial a_t^i, x_t} - \mu^i \ip{\Delta a_t^i, x_t} \\ &= \sum_{i \in I} |\mu^i|\omega_t \|x_t\|_{V_{t-1}^{-1}} - \mu^i \ip{\Delta a^i_t, x_t} \ge 0, \end{align*} and we are done.
\end{proof}

The role of the non-degeneracy condition Assumption~\ref{assumption:non_deg} in the above is fairly weak: all we really need is that $x_t$ noisily activates some index set such that the true $\theta$ can be expressed via a linear combination of the true constraint vectors of the index set. In the absence of this, the proof does not quite work as stated, since it may be the case that some constraints that are needed to express $\theta$ are not noisily activated by $x_t$ (although such constraints are activated by $x^*$). This removes the equality of the various programs we wrote, and would only leave us with a lower bound (in terms of some of these active at $x^*$ but not noisily active at $x_t$ constraints, along with the ones above), and it is unclear if $x_t$ must also optimise this lower bound. 

Nevertheless, we believe that this requirement is an artefact of our proof strategy: in general, optimistic play, when it leaks out of the safe set, has a tremendous freedom to activate any noisy constraints, and the conspiring of a choice of $\delta \theta $ and $\delta A$ that makes the point suboptimal is severely constrained due to the presence of a large number of over-efficient actions in the vicinity of the safe set. Exactly nailing down an argument that cleanly expresses this intuition is an open problem.

\subsection{Proof of the Main Theorem}\label{appx:main_theorem_proof}

With all the pieces in place, we proceed to argue our main claim.

\begin{proof}[Proof of Theorem~\ref{thm:main_regret_bound}]

With probability at least $1-\delta$, all the confidence sets are consistent. We assume that this indeed occurs, and argue the claim under this event. 

We first split the time horizon into two groups depending on whether $x_t$ noisily activates suboptimal BISs or not by defining \[ \mathfrak{T}_1 := \{ t \in [d+1:T]: \exists \textrm{ a suboptimal BIS $I$ such that } x_t \in \widetilde{\mathcal{X}}_t^I\}.\]
Notice that for $t \in [d+1:T] \setminus \mathfrak{T}_1,$ $x_t$ only activates optimal BISs. 

Now, by Lemma~\ref{lem:spread_is_finite_AND_noise_scale_lower_bound}, for all $t \in \mathfrak{T}_1,$ $\rho_t(x_t;\delta) \ge \Gamma,$ and further by the Lemma~\ref{lemma:optimally_associated_BIS_are_good}, for every $t \in [d+1:T]\setminus \mathfrak{T}_1,$ it holds that \( \ip{\theta, x^* - x_t} \le 0\). Finally, we observe that it must hold that for all times \[ \ip{\theta, x_t} \ge \ip{\theta, x^*} - \rho_t(x_t;\delta).\] Indeed, due to consistency, both $\theta$ and $x^*$ are feasible choices for the actions of \algonamenospace. Thus, if some $\tilde\theta, x_t$ are chosen instead, then $\langle \tilde\theta ,x_t\rangle \ge \ip{\theta, x^*}.$ But by Lemma~\ref{lem:noise_scale}, under consistency, $\ip{\theta, x_t} \ge \langle\tilde\theta, x_t\rangle -\rho_t(x_t;\delta)$, giving the above claim.

We thus have the efficacy control \begin{align*}
    \eff_T &= \sum_t \ip{\theta, x^* - x_t}_+ = \sum_{t \le d} \ip{\theta, x^* - x_t}_+ + \sum_{t \in \mathfrak{T}_1} \ip{\theta, x^* - x_t}_+ + \sum_{t \not\in \mathfrak{T}_1} \ip{\theta, x^* - x_t}_+ \\
    &\le d+ \sum_{t \in \mathfrak{T}_1} \rho_t(x_t;\delta) + 0 \\
    &\le d + \sum_{t} \rho_t(x_t;\delta) \indi\{ \rho_t(x_t;\delta) \ge \Gamma\} \\
    &\le d + \sum_{t} \rho_t(x_t;\delta) \cdot \frac{\rho_t(x_t;\delta)}{\Gamma} \\ &= d + \frac{1}{\Gamma}\sum_t \rho_t(x_t;\delta)^2,
\end{align*}
whence the claimed bound follows upon using Lemma~\ref{lem:rho_t_sums}. As in \S\ref{sec:lower_bound}, we have used the trick that $\indi\{u \ge v\} \le u/v$ for positive $v$. 

To control the safety behaviour, we observe that due to the property that $x_t \in \permset_t(\delta),$ there must exist some witness $\tA_t \in \matconf_t(\delta)$ such that $\tA_t x_t \le \alpha$. But, again by Lemma~\ref{lem:noise_scale} that under consistency, for every $i$, \[ \ip{\ta_t^i ,x_t} \ge \ip{a^i,x_t} - \rho_t(x_t;\delta), \] which implies that \[ \max_i \ip{\ta^i, x_t} - \alpha^i \le \rho_t(x_t;\delta).\] But then \[ \saf_T = \sum_{t \le T} \max_i (\ip{\ta^i, x_t} - \alpha^i)_+ \le \sum_{t \le T} \rho_t(x_t;\delta),\] and the claim is immediate from Lemma~\ref{lem:rho_t_sums}. Above, we have used the elementary fact that if $u \le v$ and $v > 0,$ then $(u)_+ \le v$.

We note that the upper bound in Theorem~\ref{thm:poly_upper_bound} is also immediate from the above argument. The control on $\saf_T$ can be repeated verbatim, while to control $\eff_T,$ we note that we began by showing that $\ip{\theta, x^* - x_t} \le \rho_t(x_t;\delta),$ so the conclusion of the control on $\saf_T$ above can be repeated verbatim.
\end{proof}

\subsection{Proofs of Polylogarithmic Safety Violation Claims from \S\ref{section:regret}}\label{appx:subsidiary_claims}

Finally, we show the proof of the subsidiary observation from \S\ref{section:regret}.

\subsubsection{Finite Precision in Constraint Levels}\label{appx:finite_precision_levels}

The argument relies on the following observation.

\begin{lemma}\label{lem:cannot_have_large_violation_without_large_noise}
    Under consistency, for every $\varepsilon > 0, t$ if \algonameparam plays an action $x_t$ such that $\max_i (\ip{a^i,x_t} - \alpha^i)_+ \ge \varepsilon$ then $\rho_t(x_t;\delta) \ge \varepsilon$.
\end{lemma}
\begin{proof}
        As in the proof of Theorem~\ref{thm:main_regret_bound}, if the algorithm plays $x_t$, then \[\exists \tA \in \matconf_t(\delta): \tA x_t \le \alpha. \] But, under consistency, by Lemma~\ref{lem:noise_scale}, \[ \tA x_t \ge Ax_t - \rho_t(x_t;\delta) \unk,\] and so for every $i$, \[ \ip{a,x_t} - \alpha^i \le \ip{\ta^i,x_t} + \rho_t(x_t;\delta) -\alpha^i \le \rho_t(x_t;\delta), \] and the claim follows by maximising over $i$.
    \end{proof}

The above is enough to enable the argument, which goes along the lines of the proof of logarithmic bounds on $\eff_T$ in Theorem~\ref{thm:main_regret_bound}.
\begin{proof}[Proof of Theorem~\ref{thm:finite_precision_in_levels}]
    As always, we begin by assuming consistency of the confidence sets, which occurs with probability at least $1-\delta$. Observe that the proof of efficacy can be repeated verbatim from the previous section under consistency. To control the net violations, first recall that by Lemma~\ref{lem:cannot_have_large_violation_without_large_noise}, $\exists i: \ip{a^i,x_t} - \alpha^i > \varepsilon \implies \rho_t(x_t;\delta).$ It thus follows that \begin{align*}
        \saf_T^{\varepsilon} &= \sum_{t \le T} (\ip{a^i,x_t} - \alpha^i)\indi\{\exists i: \ip{a^i,x_t} - \alpha^i > \varepsilon\} \\
                            &\le \sum_{t \le T} \rho_t(x_t;\delta) \indi\{\rho_t(x_t;\delta) > \varepsilon\} \\
                            &\le \sum \rho_t(x_t;\delta)^2/\varepsilon,
    \end{align*}
    and the claim follows from Lemma~\ref{lem:rho_t_sums}.
\end{proof}

\subsubsection{Finite Precision in Constraints}\label{appx:finite_precision_constraints}

We argue that due to the finite precision in the constraint levels, there exists a minimal error scale for the problem. 

\begin{lemma}
    There exists a constant $\pi >0$ such that if the confidence sets are consistent, and that the finite-constraint-precision version of \algonameparam picks an $x_t$ that only activates optimal BISs, but $x_t$ is either infeasible or ineffective, then $\rho_t(x_t;\delta) \ge \pi$. 
\end{lemma}
\begin{proof}
Let $I$ be an optimal BIS that $x_t$ noisily activates. This is again full rank by Assumption~\ref{assumption:non_deg}, and there exists some $\tA \in \matconf_t^{\mathsf{P}}$ such that $\tA(I) x_t = \alpha(I), \tA(I) x_t \le \alpha.$ As in the proof of Theorem~\ref{thm:main_regret_bound}, we can restrict attention to $\tA$ such that $\tA(I)$ is full-rank, since one such $I,\tA$ must exist.

Since both $\tA(I)$ is full rank, we immediately know that $x_t = \tA(I)^{-1} \alpha(I).$ But then, since there are only a finite number of possible choices for $\tA(I)$ in $\mathsf{P},$ there are only a finite number of candidate $x_t$. Let us define $x(\tA(I)) = \tA(I)^{-1} \alpha(I)$, and $\mathcal{X}(I) = \{ x(\tA(I))\}.$ Since $x^*$ is assumed to be the unique optimum, we know that for each $x \in \mathcal{X}(I),$ it must hold that $\pi(x) := \max \left\{(\ip{\theta, x^* - x}, \max_i (\ip{ a^i,x} - \alpha^i)_+\right\}$ is strictly positive, which in turn yields that \[ \pi_I := \min_{x \in \mathcal{X}(I)} \pi(x) > 0.\] Of course, we also conclude then that if $x_t$ noisily activates $I$ but is infeasible or suboptimal, then it must be at least $\pi^I$-infeasible or $\pi^I$-suboptimal, which via Lemma~\ref{lem:noise_scale} and an argument similar to that in the proof of Lemma~\ref{lem:cannot_have_large_violation_without_large_noise} yields that $\rho_t(x_t;\delta) \ge \pi_I$. 

Of course, since some optimal full rank BIS must be activated, we conclude that if $x_t$ is not the optimum, then \[\rho_t(x_t;\delta) \ge \pi := \min_{\textrm{ optimal BISs } I} \pi_I,\] and we are done.\end{proof}

Let us note that the argument above is quite crude, in that we simply take a minimum over all candidates once we establish the finitude of the set of these candidates. A more refined analysis may recover stronger local behaviour by analysing what types of $\tA(I)$ remain in $\matconf_t(\delta)$ once enough information has been accumulated, and use this to develop notions of gaps for finite-constraint-precision scenarios that dominate the quantity we have constructed above. We leave this interesting line of study for future work.

In any case, exploiting the above yields the result.

\begin{proof}[Proof of Theorem~\ref{thm:finite_precision_in_constraints}]
    Working along the lines of the proof of Theorem~\ref{thm:main_regret_bound} yields control in both the efficacy and safety costs accumulated over times $t$ for which a suboptimal BIS was activated of the form $O(\Gamma^{1} d^2 \log^2 T)$. Restricting attention then to optimal BISs, by the above, if a suboptimal or infeasible action $x_t$ were picked, then by the above Lemma, $\rho_t(x_t;\delta) \ge \pi$. This lets us repeat the same argument, but now over $t$ for which an optimal BIS was activated, which yields bounds of $O(\pi^{-1} d^2 \log^2 (T))$, and the overall costs is bounded by the sum of these two quantities. 
\end{proof}

\subsubsection{Finite Action Setting.}\label{appx:finite_arm}

Let us specify the setting in a little more detail: we are supplied with a finite set $\mathcal{A} \subset\mathbb{R}^d$, and in each round the learner chooses one action $x_t \in \mathcal{A}$. The linear reward and constraint structures are kept identical, and $x^*$ is updated to be the best action in $\mathcal{A},$ i.e., \[ x^* := \argmax \ip{\theta, x} : A x \le \alpha, x \in \mathcal{A}. \] Note that the known constraints are no longer necessary: if they are given, then we may filter $\mathcal{A}$ before play starts. The gap $\Delta := \min_{x \in \mathcal{A}, x \neq x^*} \max(\ip{\theta, x^*-x} , \max_i (\ip{a^i,x} - \alpha^i)_+)$ is non-zero simply because each suboptimal arm in $\mathcal{A}$ must be either infeasible, or ineffective, and the minimisation is over a finite set.

The result relies on the following observation, which follows straightforwardly from Lemma~\ref{lem:noise_scale}.
\begin{lemma}\label{lemma:finite_action_noise_scale_lower_bound}
    If the confidence sets are consistent, and the modified finite-action version of \algoname chooses $x_t\neq x_*$ from $\mathcal{A},$ then $\rho_t \ge \Delta$ 
\end{lemma}

\begin{proof}[Proof of Lemma \ref{lemma:finite_action_noise_scale_lower_bound}]
    Notice that the basic result Lemma~\ref{lem:noise_scale} remains valid in this setting. As a result, if the confidence sets are consistent, then since $x_t$ is permissible, there exists $\tamat \in \matconf_t$ such that \(\tamat x_t \le \alpha,\) and some $\ttheta \in \confset_t^\theta: \ip{\ttheta, x} \ge \ip{\theta, x^*}$. Further, either there exists $i : \ip{a^i,x_t} \ge \alpha^i + \Delta$ or $\ip{\theta,x} \le \ip{\theta, x^*} - \Delta$. But by consistency, $\ip{\ta^i,x_t} \ge \ip{a^i,x_t} - \rho_t(x_t;\delta)$ and $\ip{\ttheta, x} \le \ip{\theta,x} + \rho_t(x_t;\delta),$ so either case implies $\rho_t(x_t;\delta) \ge \Delta,$ which thus must hold.
\end{proof}

\begin{proof}[Proof of Proposition \ref{prop:finite_action_regret_bound}]
    The claim can be shown using Lemma~\ref{lemma:finite_action_noise_scale_lower_bound} along the lines of the proof of Theorem~\ref{thm:finite_precision_in_levels}.
\end{proof}

\section{Proofs of Lower Bounds}

We conclude by showing the lower bounds claimed in the main text.

\subsection{Proof of Polynomial Lower Bound}\label{appx:unslacked_regret_lower_bound_proof}

We argue Theorem~\ref{thm:unslacked_regret_lower_bound} by fleshing out the example developed in \S~\ref{sec:lower_bound}. The proof uses techniques that are largely standard in the bandit literature \citep[Ch.~24]{lattimore2020bandit}.

\begin{proof}[Proof of Theorem~\ref{thm:unslacked_regret_lower_bound}]

The instance we consider is 
\[\mathcal{X} = [0,1], \theta^* = 1, a^1 = (1 \pm \kappa)/2, \alpha^1 = 1/4, w_t^i\stackrel{i.i.d.}{\sim}\mathcal{N}(0,1), i \in \{0, 1\}\] 
for some $\kappa \in (0, 1/4)$. Note that implicitly, the above has the know constraints $-x \le 0$ and $x \le 1$. Of course, this one-dimensional construction can be embedded into an arbitrary dimension (for instance, by taking a very skinny box domain, and only enforcing this single unknown constraint).

In the above case, the optimal feasible solutions are $x^+ = \frac{1}{2(1+\kappa)}, x^- = \frac{1}{2(1-\kappa)}$ for these two instance respectively. In addition, both of these two instances are at least $1/8-$well separated. The key observation is the indistinguishability of these two instances with $\ll 1/\kappa^2$ actions. 

Indeed, let $\mathbb{P}^+,\mathbb{P}^-$ be the distributions induced by the two problem instances and the learning algorithm. Since in either case, the noise distribution is standard Gaussian, and the reward distributions are identical, it follows that \[ D(\mathbb{P}^+(r_t,s_t^1)\|\mathbb{P}^-(r_t, s_t^1)|x_t = x) = \frac{(\kappa x)^2}{2}\le \frac{\kappa^2}{2}, \] where we have used standard results about the KL-divergence between two Gaussians. Further, since actions must be causal, and since the noise is independent, we conclude that over the whole trajectory,  
\[D(\mathbb{P}^+(\mathcal{H}_T)\|\mathbb{P}^-(\mathcal{H}_T))\le\frac{T\kappa^2}{2}.\]

\newcommand{\xav}{x^{\mathsf{av}}}

Let $\xav := (x^+ + x^-)/2 = \frac{1}{2(1-\kappa^2)}$. Observe that 
\begin{itemize}
    \item if the ground truth is $a^1= (1+\kappa)/2$ and $x_t \ge \xav$, then the algorithm incurs an instantaneous safety violation of at least $(1+\kappa)/2 \cdot \xav - 1/4 = \frac{1+\kappa}{2} \cdot \frac{1}{2(1-\kappa^2)} - \frac{1}{4}=\frac{\kappa}{4(1-\kappa)}\ge \frac{\kappa}{4}$; 
    \item if the ground truth is $a^1= (1-\kappa)/2$ $x_t < \xav$, then the algorithm incurs an instantaneous efficacy regret of at least $\frac{1}{2(1-\kappa)} - \frac{1}{2(1-\kappa^2)}\ge \frac{\kappa}{2}$
\end{itemize}

Let $\mathsf{A}$ be the event $\{ \#\{t : x_t \ge \xav\} \ge T/2\}.$ Using the Bretagnolle-Huber inequality \citep[Thm.~14.2]{lattimore2020bandit}, 

\[\mathbb{P}^+(\mathsf{A}) + \mathbb{P}^-(\mathsf{A}^c) \ge \frac{1}{2}\exp\left(D(\mathbb{P}^+(\mathcal{H}_T)\|\mathbb{P}^-(\mathcal{H}_T))\right) \ge \frac{1}{2}\exp(-T\kappa^2/2 ).\]

Let $\eff_T^-$ denote the efficacy regret incurred by the learner under $\mathbb{P}^-$ and $\saf_T^+$ denote the safety violation incurred by the learner under $\mathbb{P}^+$. Under the event $\mathsf{A},$ if the true $a$ was $(1+\kappa)/2,$ at least $T/2$ rounds incurred a safety regret of at least $\kappa/4,$ and so $\saf_T^+ \ge \kappa T/8$. Similarly, under $\mathsf{A}^c,$ at least $T/2$ rounds had $z_t = -1,$ implying that $\eff_T^- \ge T\kappa/8$.

But this implies that \[ \max( \mathbb{E}^-(\eff_T^-), \mathbb{E}^+(\saf_T^+))  \ge \frac{T\kappa}{8} \max(\mathbb{P}^+(\mathsf{A}), \mathbb{P}^-(\mathsf{A}^c)) \ge \frac{T\kappa}{32} \exp(-T\kappa^2/2). \]

For $T \ge 16,$ we may choose $\kappa = 1/\sqrt{T} < 1/4$ to conclude that in at least one instance, the safety or efficacy regret incurred must be at least $ \sqrt{T}/(32 e^{1/2}) \ge \sqrt{T}/64.$

\end{proof}

\subsection{Necessity of Dependence on Gaps.}\label{appx:logarithmic_lower_bound}

We conclude the theoretical part of this paper by showing Theorem~\ref{lowerpd}, via a reduction to prior lower bounds on the safe multi-armed bandit problem \citep{chen2022strategies}.

The safe MAB problem is parametrised by $d$ arms with mean rewards $\mu_k$ and mean safety risks $\nu_k$ each. The optimal arm, $k^*$ has reward $\mu_*$ and the safety risk $ny_* < \alpha$. The associated efficacy and safety gaps are $\Delta_k := (\mu_* - \mu_k)_+$ and $\Gamma_k := (\nu_k - \alpha)_+.$ In each round, the leaner is required to select one arm, and observes bounded signals with the above mean for both the rewards and safety. Implicitly, this can be thought of as a linear bandit setting, with the known constraints being that $x$ lies in a simplex, the reward vector $\theta,$ and the constraint vector $a$.   This reduction, however, is not completely correct: in the safe MAB problem, the actions are required to lie entirely on the corner points of the simplex, and we are not allowed to play in the interior. While it is standard to view $x$ as a probability of selecting each arm in a MAB instance, this reduction fails due to the nonlinearity in our metrics. Indeed, the safe MAB problem considers the metrics \[ \eff_T^{\mathsf{MAB}} := \sum (\mu^* - \mu_{A_t})_+ , \saf_T^{\mathsf{MAB}} := \sum (\nu_{A_t} - \alpha)_+.\] As a result, if the optimum of the SLB problem lies away from the corner points of the simplex, then the SLB problem can incur low regret, while the corresponding MAB actions would incur linear regret. Nevertheless, we shall argue below that for carefully designed instances, a low regret in the linear bandit problem does ensure nontrivial regret in the safe MAB problem.

The main result we shall use is the following, which is a mild variation of Proposition 6 of \citet{chen2022strategies}, and can be shown using their proof. 
\begin{lemma}\label{lemma:mab_lower_bound}
    Let $f:\mathbb{N} \to [0,\infty)$ be any function fixed function such that $f(T) \le T$ for all $T$. If an algorithm ensures that for every safe MAB instance, suboptimal arms are not played more than $f(T)$ times in expectation, then for every $\theta, a,$ there exists a choice of arm distributions for the safe MAB instance for which the means are as described, and the number of times each suboptimal arm $k$ is played is lower bounded in expectation as \[ \mathbb{E}[N_T^k] \ge \frac{1}{ (d(\mu_k \| \mu_*) \indi\{\mu_k < \mu_*\} + d(\nu_k\|\alpha) \indi\{\nu_k > \alpha\})} \cdot \left(  (1-f(T)/T) \log \frac{T}{f(T)} - \log(2)\right),\] where $d(u\|v)$ is the KL divergence between Bernoulli laws with means $u$ and $v$. In particular, these distributions are simply Bernoulli laws with the above means.
\end{lemma}

Our argument for the linear bandit proceeds thus. We shall carefully design a safe linear bandit instance for which we essentially provide multi-armed bandit feedback by using the standard reduction that each coordinate of $x_t$ represents the probability of pulling the corresponding arm. We shall show that in the selected instance, achieving low linear regret ensures that the MAB regret is controlled (although to a weaker level). Then exploiting the above lower bound, we shall argue that the regret of the safe linear bandit cannot be too good, since it would violate the above lower bound.

\begin{proof}[Proof of Theorem~\ref{lowerpd}] We first carefully describe our main constructions for the SLB and MAB, form a crude bound that allows us to use Lemma~\ref{lemma:mab_lower_bound}, and then refine the analysis to show effective lower bounds on the SLB regret.

\paragraph{SLB Instance.} We work with $d = 2$ with a single unknown constraint. Let $\theta = (\theta_1, \theta_2)$ and $a^1 = (\alpha, a^1_2)$ be vectors in $[0,1]^2$ such that $\theta_2 > \theta_1 > 0, a^1_2 > \alpha > 0$ and $\theta_2 \alpha < \theta_1 a^1_2.$ The safe bandit instance we design is \[ \max\langle \theta,x\rangle : x_1 \ge 0, x_2\ge 0, x_1 + x_2 \le 1,\langle a,x\rangle \le \alpha,\] where the last constraint is unknown and the rest are known. Let us call the three known constraints $a^2, a^3, a^4$. There are 6 BISs, with the associated points and gaps shown in Table~\ref{table:lower_bound_BISs} below. The only points meeting the constraints $a^3$ and $a^4$ is $(0,1)$, which is infeasible. Note that the situation is highly degenerate at the optimal point is $x^* = (1,0)$ and three distinct BISs activate it. Nevertheless, each of these BISs is full rank. Further, since the algorithm ensures that $\saf_T$ and $\eff_T$ are both $O(\sqrt{T})$ in general, our discussion below is effective. 

\begin{table}[ht]
\centering
\caption{Description of BISs in our construction.}\label{table:lower_bound_BISs}
\begin{tabular}{c|c|c|c}
BIS & Activating Point & $\zeta_*(I)$ & $\eta_*(I)$ \\ \hline
$\{1,2\}$ &   $(0,\alpha/a^1_2 )$ & 0 & $(\theta_1 a^1_2 - \alpha\theta_2)/(a^1_2 + \theta_2)$ \\
$\{1,3\}$ &  $(1,0)$  & 0  & 0 \\
$\{1,4\}$ &  $(1,0)$ & 0 & 0\\
$\{2,3\}$ & $(0,0 )$ &0 & $\theta_1$  \\
$\{2,4\}$ & $(1,0)$ &0 &0 \\ 
$\{3,4\}$ & $\emptyset$ & $a^1_2 - \alpha$ & 0
\end{tabular}
\end{table}

The gap of this instance is \[ \Gamma := \min \left( \theta_1, a^1_2 - \alpha, \frac{\theta_1 a^1_2 - \alpha \theta_2}{a^1_2 + \theta_2}\right).\] Our construction requires that this is at least $\Omega(\min(\theta_1, a^1_2 - \alpha))$. This can always be ensured, example, by using the parameterisation $\theta_2 = 2\theta_1, a^1_2 =  4\theta_1, \alpha = \theta_1/2,$ whence the expressions work out to \[ a^1_2 - \alpha = 7\theta_1/2, \frac{\theta_1 a^1_2 - \alpha \theta_2}{a^1_2 + \theta_2} = 3\theta_1/5,\] giving us $\Gamma \ge \theta_1/2.$ We further impose the condition $4\theta_1 < 1/4.$ Thus, this instance lets us express every value of $\Gamma < 1/32.$ 

\paragraph{Safe MAB Instance.} Let us now describe the associated MAB instance. We work with three arms of means $\mu = (\nicefrac12 + \theta_1,\nicefrac12 + \theta_2,\nicefrac12)$ and risks $(\nicefrac12+ \alpha, \nicefrac12 + a^1_2, \nicefrac12)$. In each case, the underlying laws are Benoullis with the associated mean, all taken to be independent, which forms the family of instances that underly Lemma~\ref{lemma:mab_lower_bound}. The connection to the linear bandit instance is as follows: each time we pick $(x_1,x_2),$ we sample a random variable in $\{1,2,3\}$ according to the pmf $(x_1, x_2, 1-x_1-x_2)$, pull the corresponding arm, and then supply the resulting rewards and risks with $\nicefrac12$ subtracted to the linear bandit instance. Note that this is an unbiased measurement of the mean for the linear bandit, since \[ \mathbb{E}[R] = x_1 \cdot (\nicefrac12 +\theta_1) + x_2 \cdot \nicefrac12 + \theta_2) + x_2 \cdot (\nicefrac12) - \nicefrac12 = x_1\theta_1 + x_2\theta_2 \] and similarly for the safety risk. These $\nicefrac12$ are added to ensure because then the KL divergences appearing in the bound of Lemma~\ref{lemma:mab_lower_bound} take the form $d(\nicefrac12 \|\nicefrac12 + \theta_1)$ and $d( \nicefrac12 + a^1_2 \| \nicefrac12 + \alpha)$, and the arguments are bounded away from $0$ and $1$, ensuring that the behaviour for small $\theta_1$ is quadratic rather than the potentially worse dependence near $0$ and $1$. To ensure this good behaviour, we use that $a^1_2 < 1/4,$ due to which $a^1_2 + 1/2 < 13/14$ is bounded away from $1$, which is the origin of our condition $\theta_1 \le 1/16$ in the previous paragraph. The key observation is that in the safe MAB instance, $\mathbb{E}[N_T^2] = \sum x_{t,2}$ and $\mathbb{E}[N_T^3] = \sum (1 - x_{t,1} - x_{t,2}),$ where $x_{t,k}$ is the $k$th component of $x_t$.

\paragraph {Crude Bound.} We first show that as long as the algorithm ensures that $\max(\eff_T, \saf_T) = O( T^{1-c}),$ the play of suboptimal arms in the MAB instance is at least $\Omega(\theta_1^{-2} \log T)$. Fix $\theta_1$ and the above models. Suppose that the safe linear bandit ensures that $\eff_T \le g(T)$ and $\saf_T \le g(T)$ for every instance, where $g(T) \le T$ is an arbitrary monotonic function. Let $\zeta > 0$ be a parameter that we will fix later. Then observe that if the linear bandit instance ever plays a point $(x_1,x_2)$ such that \[ \langle a^1,x\rangle \ge \alpha + \zeta \quad \textrm{or} \quad \langle\theta, x\rangle \le \theta_1 -\zeta,\] then it would incur a point wise cost of at least $\gamma$ in the round, for either $\eff_T$ or $\saf_T$. This means that the number of rounds in which it plays such points is bounded as $ g(T)/\zeta.$  So, in at least $\max(T - g(T)/\zeta, 0)$ rounds, the safe linear bandit instance plays in the region \[ P_\zeta := \{ \ip{a^1,x} \le \alpha + \zeta , \ip{\theta, x} \ge \theta_1 - \zeta, x_1 \ge 0, x_2 \ge 0, x_1 + x_2 \le 1 \}. \] Now notice that both $x_2$ and $x_1 + x_2$ are upper bounded in this region. Indeed, the corner points of this region are \[ \left(1-\frac\zeta\theta_1, 0\right), \left(1-\frac\zeta{a^1_2 - \alpha}, \frac{\zeta}{a^1_2 - \alpha}\right),  \left( 1- \frac{\zeta}{\theta_1} \left\{ 1 + \frac{\theta_2(\theta_1 + \alpha)}{\theta_1(\theta_1 a^1_2 - \alpha \theta_2)}\right\}, \frac{\theta_1 + \alpha}{\theta_1 a^1_2 - \alpha \theta_2} \frac{\zeta}{\theta_1}\right), (1,0),\]
 and so ensuring that $a^1_2 - \alpha, \theta_1 a^1_2 - \alpha \theta_2 = \Omega(\theta_1),$ we have \[x \in P_\zeta \implies x_2 \le \zeta/\theta_1, (1-x_1 - x_2)  = O(\zeta/\theta_1). \] Of course, outside of $P_\zeta,$ $x_2 \le 1, 1- x_1 - x_2 \le 1$. The calculation holds no matter the $\zeta$ we chose so long as $\zeta \ll \theta_1$. This means that for every $\zeta = O(\theta_1),$ \begin{align*} \mathbb{E}[N_T^2] &= \sum \mathbb{E}[x_{t,2}] \le O(\zeta/\theta_1) T + \frac{g(T)}{\zeta} \\ \mathbb{E}[N_T^3] &= \sum \mathbb{E}[(1- x_{t,1} - x_{t,2})] \le O(\zeta/\theta_1) T + \frac{g(T)}{\zeta},\end{align*} i.e., we have shown that the safe MAB incurs regret bounds of at most $f(T) = O(\zeta T) + g(T)\theta_1/\zeta$ for both $\eff_T$ and $\saf_T$. 

 Since $g(T)  \le C T^{1-c}$ for some constants $C, c,$ by taking $\zeta = T^{-c/2},$ for large enough $t$, we thus have the low-regret bound $\max( \mathbb{E}[\eff_T^{\mathsf{MAB}}] , \mathbb{E}[\saf_T^{\mathsf{MAB}}]) \le C T^{1-c/2}.$ But then, by Lemma~\ref{lemma:mab_lower_bound}, it must follow that as $T \to \infty,$ \begin{align*}
    \mathbb{E}[N_T^2] \ge \frac{1}{d(\nicefrac12 + 4\theta_1\| \nicefrac12 + \theta_1/2) }\left( (1-o(1)) \frac{c}{2} \log T - O(1) \right)  = \Omega(\theta_1^{-2} \log T), \\
    \textrm{or } \mathbb{E}[N_T^3] \ge \frac{1}{d(\nicefrac12\|\nicefrac12 + \theta_1 )} \left( (1-o(1)) \frac{c}{2} \log T - O(1) \right) = \Omega(\theta_1^{-2} \log T)
 \end{align*} 

To use these bounds effectively, we employ a computer algebra system to argue that\footnote{Observe that since the divergences considered are minimised to $0$ at $\theta_1 = 0,$ the local behaviour for small $\theta_1$ is quadratic. Further, the function is smooth in $\theta_1$. Thus, for large enough $K$, there exists an interval $[0,\theta_1(K)]$ such that for any $x$ in this region, $d(1/2 + u x\|1/2 + vx) \le Kx^2$. We simply plugged in various constants for $K$ until we found that $\theta_1(27) \ge \nicefrac{1}{16}$.} \[ \forall \theta_1 \le \nicefrac1{16}, d(\nicefrac12 + 4\theta_1\| \nicefrac12 + \theta_1/2) \le 27 \theta_1^2, d(\nicefrac12\|\nicefrac12 + \theta_1 ) \le 27\theta_1^2.\] Concretely, then the bounds above yield \[ \mathbb{E}[N_T^2 + N_T^3] \ge \frac{c}{27\theta_1^2} \left( (1-o(1)) \log T - \frac1c\log(4)\right), \] where the $o(1)$ term is $C/T^{c/2}$.

Note again that this bound is effective for our case since the method \algoname does achieve $\max(\eff_T, \saf_T) = \tilde{O}(\sqrt{T})$ with high probability, in which case we can set $c = \nicefrac12+\gamma$ for any $\gamma > 0$ in the above.

\paragraph{Lower Bounds on SLB.} Let us now come to showing the claims. We select the instance $\theta_2 = 2\theta_1, a^1_2 = 4\theta_1, \alpha = \theta_1/2$. Notice that in this case, the gaps are $(\theta_1, 7\theta_1/2, 3\theta_1/5)$, and so $\Gamma \ge \theta_1/2$. Further, $\theta_1 a^1_2 - \alpha \theta_2 = 3\theta_1,$ and so the claim on $P_\zeta$ above remains valid for all $\zeta \le \theta_1,$ and so against this instance, the above lower bounds on \( \mathbb{E}[N_T^2] + \mathbb{E}[N_T^3] \).

But, observe that for any choice of $x_1,x_2,$ it holds that the instantaneous efficacy regret and safety violations are  \begin{align*} (\theta_ 1- \theta_1 x_1 - \theta_2 x_2)_+ &= ( \theta_1(1-x_1 - x_2) - (\theta_2 - \theta_1)x_2 )_+ = \theta_1 ((1-x_1 - x_2) - x_2)_+\\ 
(\alpha x_1 + a^1_2 x_2 - \alpha)_+ &= ((a^1_2 - \alpha) x_2 - \alpha(1-x_1 - x_2))_+ = \frac{\theta_1}{2} (7 x_2 - (1-x_1 - x_2))_+\end{align*}

But notice that the only way both of these quantities are $0$ is if $x_2 \ge (1-x_1 - x_2) \ge 7x_2 \implies x_2 = 1-x_1 - x_2 = 0 \iff x_1 = 1$. So, in any round such that $x_1 \neq 1,$ at least one of these quantities is nonzero. More quantitatively, we have 
\begin{align*} \mathbb{E}[\eff_T] + \mathbb{E}[\saf_T] &\ge \sum \mathbb{E}[(\theta_1 - \theta_1x_{t,1} - \theta_2x_{t,2}) + (\alpha x_{t,1} - a^1_2 x_{t,2} - \alpha)] \\ &\ge \theta_1 \sum \frac{5}{2} \mathbb{E}[x_{t,2}] + \frac{1}{2} \mathbb{E}[(1-x_{t,1}  -x_{t,2})] \\ &\ge \frac{ \theta_1}{2} (\mathbb{E}[N_T^2] + \mathbb{E}[N_T^3]) \ge \frac{c (1-o(1)}{54 \theta_1} \log(T) - O(1),\end{align*} which yields the result upon recalling that $\theta_1 \ge \Gamma \ge \theta_1/2$.

\end{proof}

\section{Alternative Safety Metrics}\label{appx:alternate_relaxations}

We briefly investigate the behaviour of alternative safety metrics of the form \[ \saf_T^f := \sum_{t \le T} f(\max_i (\ip{a^i,x_t}-\alpha^i)_+), \] where $f$ is some increasing $h$-H\"{o}lder continuous map such that $\lim_{x \searrow 0}f(x) = 0$. Note that this section should be read \emph{after} the reader is familiar with our typical proof techniques.

Note that due to our assumption that $\|a^i\|\le 1,\|x\| \le 1,$ it follows that $(\ip{a^i,x_t} - \alpha^i)_+ \le 2:$ since the problem is feasible, and $|\ip{a^i,x^*}| \le 1,$ it follows that $ -\mathbf{1} \le Ax^* \le \alpha,$ and so $\ip{a^i,x} - \alpha^i \le 1 - (-1)$. Thus, only the behaviour of $f$ over $[0,2]$ matters.

Now, since $f$ is H\"{o}lder continuous, and $f(0^+) = 0,$ its behaviour near $0$ is as $f(x) \le C x^{h}$. In this case, we may as well study the behaviour of $f_h := x \mapsto x^h,$ which we argue determines the lower and upper bounds. Before proceeding, note that we may uniformly bound the noise scale by $2$, since we know that roundwise inefficacy or constraint violation can be at most $2$. 

Now, for the penalty $\saf_T^h = \sum_{ t\le T} (\max_i \ip{a^i,x_t} - \alpha^i)_+^h,$ observe that if $h \ge 2,$ we can direclty bound the behaviour of $\saf_T^h$ using \[ \saf_T^h \le \sum \rho_t^{h}(x_t;\delta) \le 2^{h-2} \sum \rho_t^{2}(x_t;\delta) \le 2^{h-2} \cdot O(d^2 \log^2 T). \] Thus, the only interesting behaviour is when $h < 2$. 

For $h \in (0,2),$ applying H\"{o}lder's inequality with $p = 2/h>1,$ we have \[\saf_T^h \le \sum \rho_t^{h}(x_t;\delta) \le \left( \sum_{t} (\rho_t^h)^{2/h}\right)^{h/2} \cdot (\sum 1^{2/(2-h)})^{1-h/2} = T^{1-h/2} \cdot O(d^{h} \log^{h} T).\] We now note that modifying the analysis of \S\ref{sec:lower_bound}, this rate of safety decay is tight. Indeed, our construction in that section shows that for $t \le 1/\kappa^2,$ one either incurs a roundwise inefficacy of $\kappa$ or a roundwise violation of $\kappa$. Accounting for the power-cost, we get a lower bound of the form \[ \textit{either } \eff_T \gtrsim \kappa \cdot \min(\kappa^{-2}, T)  \textit{ or } \saf_T^h \gtrsim \kappa^h \cdot \min(\kappa^{-2}, T).\] But again, taking $\kappa = T^{-1/2}$, we find that \[ \textit{either } \eff_T \ge \sqrt{T} \textit{ or } \saf_T^h \ge T^{1-h/2}.\] Thus, up to polylog terms, the behaviour of \textsc{doss} remains tight in terms of the $\saf_T^h$ behaviour, simultaneously for every $h > 0$.

Coming back to general smooth losses, we immediately note that the same analysis extends to any loss that is $h$-H\"{o}lder: using the bound $f(x) \le C x^h$, \[ \saf_T^f \le C \sum_t \rho_t^h(x_t;\delta),\] and the bound follows. This extends to losses $f$ that are smooth in some interval near $0^+$ of the form $(0,k)$. For the upper bound, we may decompose the net violation as \[\saf_T^f \le \sum_{t} f(\rho_t)\mathbf{1}\{\rho_t > k\} + \sum_{t} f(\rho_t) \mathbf{1}\{\rho_t \le k\}.\] The latter term can be dealt with as above, since $f(x) \le C x^h$ over $(0,k),$ while the former term can be bounded as \[ \sum_{t} f(\rho_t) \indi\{\rho_t \ge k\} \le (\max_{x \in [0,2]} f(x)) \sum_{t} \rho_t^2/k^2 = O( k^{-2} d^2 \log^2 T),\] leading to an additive polylogarithmic overhead beyond the main term. The lower bound also generalises: if on $(0,k)$, $f$ is $h$-H\"{o}lder but not $h'$-H\"{o}lder for any $h' > h,$ then there exists some interval $(0,k')$ over which $f/x^h$ remains both lower and upper bounded, and we can employ our lower bound for $\saf_T^h$ for $T\gg 1/(k')^2.$

\end{document}